\documentclass[twoside]{article}

\usepackage[accepted]{aistats2019}

\usepackage[utf8]{inputenc} 
\usepackage[T1]{fontenc}    
\usepackage{hyperref}       
\usepackage{url}            
\usepackage{booktabs}       
\usepackage{amsfonts}       
\usepackage{nicefrac}       
\usepackage{microtype}      

\usepackage{amssymb}
\usepackage{amsmath}
\usepackage{lmodern}
\usepackage{enumerate}
\usepackage{amsthm}
\usepackage[usenames]{color}
\usepackage{capt-of}
\usepackage{graphicx} 
\usepackage{epstopdf}
\usepackage{wrapfig}
\usepackage[font=scriptsize,labelfont=bf]{caption}
\usepackage{multirow}
\usepackage{caption}
\usepackage{algorithm}
\usepackage{algorithmic}

\usepackage{subfig}
\usepackage{tabularx}
\usepackage{afterpage}
\usepackage{floatflt}
\usepackage{booktabs,colortbl}
\usepackage{arydshln}

\pdfoutput=1

\newtheorem{lemma}{Lemma}[section]

\newtheorem{theorem}{Theorem}[section]

\newtheorem{definition}{Definition}[section]

\newcommand{\inner}[1]{\left\langle#1\right\rangle}

\def\R{\mathbb{R}}

\newcommand{\norm}[1]{\left\|#1\right\|}

\def\argmax{\mathop{\rm arg\,max}\limits}
\def\minop{\mathop{\rm min}\limits}
\def\maxop{\mathop{\rm max}\limits}
\def\sign{\mathop{\rm sign}\limits}

\def\min{\mathop{\rm min}\nolimits}
\def\max{\mathop{\rm max}\nolimits}

\makeatletter
\def\blfootnote{\xdef\@thefnmark{}\@footnotetext}
\makeatother

\newif\ifpaper
\papertrue

\newif\ifappendix
\appendixtrue

\newif\ifnotinapp
\notinappfalse

\title{Robust ReLU networks via Maximization of Linear Regions}

\author{
  Francesco Croce^{1,2} and Matthias Hein^{2}\\ and Maksym Andriushchenko\\
  Department of Mathematics and Computer Science\\
  Saarland University, Saarbr{\"u}cken Informatics Campus, Germany\\
}

\begin{document}

\twocolumn[

\aistatstitle{Provable Robustness of ReLU networks via Maximization of Linear Regions}

\aistatsauthor{ Francesco Croce*$^{,1}$ \And Maksym Andriushchenko*$^{,2}$ \And  Matthias Hein$^1$ }
\runningauthor{Francesco Croce, Maksym Andriushchenko,  Matthias Hein}

\aistatsaddress{ $^1$University of T\"ubingen, Germany \hspace{1mm} $^2$Saarland University, Germany 
}
]

\begin{abstract}
It has been shown that neural network classifiers are not robust. This raises concerns
about their usage in safety-critical systems. We propose in this paper a regularization scheme for ReLU networks
which provably improves the robustness of the classifier by maximizing the linear regions of the classifier as well
as the distance to the decision boundary. Our techniques allow even to find the minimal adversarial perturbation
for a fraction of test points for large networks. In the experiments we show that our approach improves upon
adversarial training both in terms of lower and upper bounds on the robustness and is comparable or better than
the state-of-the-art in terms of test error and robustness.
\end{abstract}

\section{Introduction} 

In recent years it has been highlighted that state-of-the-art neural networks are highly non-robust: small changes to an input image, which are almost
non-perceivable for humans, change the classifier decision and the wrong decision has even high confidence \cite{SzeEtal2014, GooShlSze2015}. This
calls into question the use of neural networks in safety-critical systems e.g. medical diagnosis systems or self-driving cars and opens up possibilities
to actively attack an ML system in an adversarial way \cite{PapEtAl2016a,LiuEtAl2016,KurGooBen2016a,LiuEtAl2016}. Moreover, this non-robustness has also implications on follow-up processes like interpretability.
How should we be able to interpret classifier decisions if very small changes of the input lead to different decisions? 

The finding of the non-robustness
initiated a competition where on the one hand increasingly more sophisticated attack procedures were proposed 
\cite{GooShlSze2015,HuaEtAl2016,MooFawFro2016,CarWag2016} and on the other hand research was focused to develop 
stronger defenses  against these attacks \cite{GuRig2015,GooShlSze2015,ZheEtAl2016,PapEtAl2016a,HuaEtAl2016,BasEtAl2016,MadEtAl2018}. In the end it turned out that for many proposed defenses there exists still a way to attack the classifier successfully \cite{CarWag2017, AthEtAl2018, MosEtAl18}, with the notable exception of \cite{MadEtAl2018} which was shown to be empirically robust wrt the $l_\infty$-norm. Thus considering the high importance of robustness in 
safety-critical machine learning applications, we need robustness guarantees, where one can provide for each test point the radius of a ball on which the
classifier will not change the decision and thus no attack whatsoever will be able to create an adversarial example inside this ball.\\
Therefore recent research has focused on such provable guarantees of a classifier with respect to the $l_1$-norm \cite{CarEtAl2017}, $l_2$-norm \cite{HeiAnd2017} and $l_\infty$-norm \cite{BasEtAl2016, KatzEtAl2017,  RagSteLia2018,    TjeTed2017, WonKol2018,WengEtAl2018}. Some works try to solve the combinatorial problem of computing for each test instance the norm of the minimal perturbation necessary to change the decision \cite{CarEtAl2017, KatzEtAl2017, TjeTed2017}. Unfortunately, these approaches do not scale with normal training to large networks. Another line of research thus focuses on lower bounds on the norm of the minimal perturbation necessary to change the decision
\cite{HeiAnd2017,RagSteLia2018,WonKol2018,WengEtAl2018}. 
\blfootnote{*Equal contribution.}

Moreover, in recent years several new ways to regularize neural networks \cite{SokEtAl2016,CisEtAl2017} or new forms of losses \cite{ElsEtAl2018} have been proposed with the idea of enforcing a large margin, that is a large distance between training instances and decision boundaries. However, these papers do not directly optimize a robustness guarantee. In spirit our paper is closest to \cite{HeiAnd2017,RagSteLia2018,WonKol2018,MirGehVec2018}. All of them are aiming at providing robustness guarantees and at the same time they propose a new way how one can optimize the robustness guarantee during training. Currently, up to our knowledge only \cite{WonKol2018} can optimize robustness wrt to multiple $p$-norms, whereas \cite{HeiAnd2017} is restricted to $l_2$ and \cite{RagSteLia2018,MirGehVec2018} to $l_\infty$.

In this paper we propose a regularization scheme for the class of ReLU networks (feedforward networks with ReLU activation functions including convolutional and residual architectures with max- or sum-pooling layers) which provably increases the robustness of classifiers. We use the fact that these networks lead to continuous piecewise affine classifier functions and show how to get either the optimal minimal perturbation or a  lower bound using the properties of the linear region in which the point lies. As a result of this analysis we propose a new regularization
scheme which directly maximizes the lower bound on the robustness guarantee. This allows us to get classifiers with good test error and good
robustness guarantees at the same time. While we focus on robustness with respect to $l_2$- and $l_\infty$-distance, our approach applies to all $l_p$-norms.
Finally, we show in experiments on four datasets that our approach improves lower bounds as well as upper bounds on the robust test error and can be successfully integrated with adversarial training \cite{GooShlSze2015,MadEtAl2018}. Our main observation is that our proposed regularizer significantly improves provable robustness evaluated with the certification method of \cite{WonKol2018}, and especially with the combinatorial solver of \cite{TjeTed2017}. We note that at the same time, models with adversarial training alone are not certifiable, i.e. they often have vacuous upper bounds on the robust test error.


\section{Local properties of ReLU networks}\label{sec:explicit}
Feedforward neural networks which use piecewise affine activation functions (e.g. ReLU, leaky ReLU) and are linear in the output layer can be rewritten as continuous piecewise affine functions \cite{AroEtAl2018,CroHei18}.
\begin{definition} A function $f:\R^d \rightarrow \R$ is called \emph{piecewise affine} if there exists a finite set of polytopes $\{Q_r\}_{r=1}^M$ 
	(referred to as \emph{linear regions} of $f$) such that $\cup_{r=1}^M Q_r= \R^d$ and $f$ is an affine function when restricted to  every $Q_r$.
\end{definition}

In the following we introduce the notation required for the explicit description of the linear regions and decision boundaries of a multi-class ReLU classifier $f:\R^d \rightarrow \R^K$ (where $d$ is the input space dimension and $K$ the number of classes) which has no activation function in the output layer. The decision of $f$ at a point $x$ is given by $\argmax_{r=1,\ldots,K} f_r(x)$. The discrimination between linear regions and decision boundaries allows us to share the linear regions across classifier components.

We denote by $\sigma:\R \rightarrow \R$, $\sigma(t)=\max\{0,t\}$, the ReLU activation function, $L+1$ is the number of layers and $W^{(l)} \in\mathbb{R}^{n_l \times n_{l-1}}$ and $b^{(l)} \in \R^{n_l}$ respectively are the weights and offset vectors of layer $l$,  for $l=1,\ldots,L+1$ and $n_0=d$. If $x\in \R^d$ and $g^{(0)}(x)=x$ we define recursively the pre- and post-activation output of every layer as
\begin{gather*}    f^{(k)}(x)=W^{(k)}g^{(k-1)}(x)+b^{(k)}, \quad \mathrm{ and }\\g^{(k)}(x)=\sigma(f^{(k)}(x)), \quad k=1,\ldots,L,\end{gather*}
so that the resulting classifier is obtained as $f^{(L+1)}(x)=W^{(L+1)}g^{(L)}(x) + b^{(L+1)}$.

We derive, following \cite{CroHei18,CroRauHei2019}, the description of the polytope $Q(x)$ in which $x$ lies and the resulting affine function when $f$ is restricted to $Q(x)$. We assume for this that $x$ does not lie on
the boundary between two polytopes (which is almost always true as the faces shared by two or more polytopes have dimension strictly smaller than $d$).
Let $\Delta^{(l)},\Sigma^{(l)}\in \R^{n_l \times n_l}$ for $l=1,\ldots,L$ be diagonal matrices defined elementwise as
\begin{gather*} \Delta^{(l)}(x)_{ij} = \begin{cases} \sign(f_i^{(l)}(x)) & \textrm{ if } i=j,\\ 0 & \textrm{ else.} \end{cases}, \\
\Sigma^{(l)}(x)_{ij} = \begin{cases} 1 & \textrm{ if } i=j \textrm{ and } f_i^{(l)}(x)>0,\\ 0 & \textrm{ else.} \end{cases}.\end{gather*}
This allows us to write $f^{(k)}(x)$ as composition of affine functions, that is
\[\begin{split} f^{(k)}(x)=&W^{(k)}\Sigma^{(k-1)}(x)\Big(W^{(k-1)} \Sigma^{(k-2)}(x)\\ & \times \Big(\ldots  \Big( W^{(1)}x + b^{(1)}\Big) \ldots\Big)+ b^{(k-1)} \Big) + b^{(k)}, \end{split}\]
We can further simplify the previous expression as $f^{(k)}(x) = V^{(k)}x + a^{(k)}$, with $V^{(k)} \in \R^{n_k \times d}$ and $a^{(k)} \in \R^{n_k}$ given by
\begin{gather*} V^{(k)} = W^{(k)}\Big( \prod_{l=1}^{k-1} \Sigma^{(k-l)}(x)W^{(k-l)}\Big) \quad \mathrm{ and }\\ a^{(k)} = b^{(k)} + \sum_{l=1}^{k-1} \Big(\prod_{m=1}^{k-l} W^{(k+1-m)} \Sigma^{(k-m)}(x)\Big) b^{(l)}. \end{gather*}
Note that a forward pass through the network is sufficient to compute $V^{(k)}$ and $b^{(k)}$ for every $k$, which results in only a small overhead compared to the usual effort necessary to compute the output of $f$ at $x$.
We are then able to characterize the polytope $Q(x)$ as intersection of $N=\sum_{l=1}^L n_l$ half spaces given by 
\[\Gamma_{l,i}=\big\{z \in \R^d \,\Big|\, \Delta^{(l)}(x)\big(V_i^{(l)}z +a_i^{(l)}\big)\geq 0\big\},\]
for $l=1,\ldots,L$, $i=1,\ldots,n_l$, namely
\[Q(x)=\bigcap_{l=1,\ldots,L}\bigcap_{i=1,\ldots,n_l} \Gamma_{l,i}. \]
Note that $N$ is also the number of hidden units of the network.
Finally, we can write 
\[ \left.f^{(L+1)}(z)\right|_{Q(x)}=V^{(L+1)}z + a^{(L+1)},\] which represents the affine restriction of $f$ to $Q(x)$.
One can further distinguish the subset $Q_c(x)$ of $Q(x)$ assigned to a specific class $c$, among the $K$ available ones, which is given by
\[\begin{split} Q_c(x)= \bigcap_{\stackrel{s=1,...,K}{s\neq c}} \big\{& z \in \R^d \,\big|\,\inner{V^{(L+1)}_c -V^{(L+1)}_s, z } \\ &+ a^{(L+1)}_c - a^{(L+1)}_s \geq 0\big\} \cap Q(x), \end{split} \]
where $V^{(L+1)}_r$ is the $r$-th row of $V^{(L+1)}$. The set $Q_c(x)$ is again a polytope as it is an intersection of polytopes and it holds $Q(x)=\bigcup_{c=1,...,K}Q_c(x)$.
We refer to \cite{CroRauHei2019} how to integrate other operations/layer types e.g. max-pooling, residual and dense nets or other piecewise linear activation functions
like leaky ReLU.

\section{Robustness guarantees for ReLU networks}\label{sec:guar}
In the following we first define the problem of the minimal adversarial perturbation and then derive robustness guarantees for ReLU networks.
We call a decision of a classifier $f$ robust at $x$ if small changes of the input do not alter the decision.  
Formally, this can be described as optimization problem \eqref{eq:advopt} \cite{SzeEtal2014}. 
If the classifier outputs class $c$ for input $x$,
assuming a unique decision,  
the \emph{robustness} of $f$ at $x$ is given by
\begin{align}\label{eq:advopt}
\begin{split} \minop_{\delta \in \mathbb{R}^d} \; \norm{\delta}_p, \quad \textrm{s.th.}  \quad    &\maxop_{l\neq c} \; f_l(x+\delta) \geq f_c(x+\delta)\\ & x+\delta \in C, \end{split}
\end{align}
where $C$ is a constraint set which the generated point $x+\delta$ has to satisfy, e.g., an image has to be in $[0,1]^d$. 
The complexity of the optimization problem \eqref{eq:advopt} depends on the classifier $f$, but it is typically non-convex, see \cite{KatzEtAl2017} for a hardness result for neural networks.
\\
For standard neural networks $\norm{\delta}_p$ is very small for \emph{almost any} input $x$ of the data generating distribution, which questions the use of such classifiers in safety-critical systems. The solutions of \eqref{eq:advopt}, $x+\delta$, are called \emph{adversarial samples}.

For a linear classifier, $f(x)=Wx+b$, with $C=\mathbb{R}^d$ one can compute the solution of \eqref{eq:advopt} in closed form \cite{HuaEtAl2016}
\[ \norm{\delta}_p = \minop_{s\neq c} \frac{|\inner{w_c-w_s,x}+b_c-b_s|}{\norm{w_c-w_s}_q},\]
where $\norm{\cdot}_q$ is the dual norm of $\norm{\cdot}_p$, that is $\frac{1}{p}+\frac{1}{q}=1$. In \cite{HeiAnd2017} it has been shown that box constraints $C=[a,b]^d$ can be integrated for linear classifiers which results in simple convex optimization problems. 
In the following we use the intuition from linear classifiers and the particular structures derived in Section \ref{sec:explicit} to come up with robustness guarantees, that means lower bounds on the optimal solution of \eqref{eq:advopt}, for ReLU networks.
Moreover, we show that it is possible to compute the minimal perturbation for some of the inputs $x$  even though the general problem is NP-hard \cite{KatzEtAl2017}.

Let us start analyzing how we can solve efficiently problem \eqref{eq:advopt} inside each linear
region $Q(x)$. We first need the definition of two important quantities:
\begin{lemma}\label{le:bd-polytope}
The $l_p$-distance $d_B(x)=\min_{z \in \partial Q(x)} \norm{z-x}_p$ of $x$ to the boundary of the polytope $Q(x)$   is given by 
\begin{align*} d_B(x) = \minop_{l=1,\ldots,L}  \minop_{j=1,\ldots,n_l} \frac{\big|\inner{V^{(l)}_j,x}+a^{(l)}_j\big|}{\norm{V^{(l)}_j}_q},
\end{align*}
where $V_j^{(l)}$ is the $j$-th row of  $V^{(l)}$ and $\norm{\cdot}_q$ is the dual norm of $\norm{\cdot}_p$ ($\frac{1}{p}+\frac{1}{q}=1$).
\end{lemma}

\ifpaper
\begin{proof}
Due to the polytope structure of $Q(x)$ it holds that $d_B(x)$ is the minimum distance to the hyperplanes, $\inner{V^{(l)} _j,z}+a^{(l)}_j=0$, which constitute the boundary
of $Q(x)$. It is straightforward to check that the minimum $l_p$-distance of a hyperplane $H=\{z \,|\,\inner{w,z}+b=0\}$ to $x$ is given by $\frac{|\inner{w,x}+b|}{\norm{w}_q}$ where $\norm{\cdot}_q$ is the
dual norm. This can be obtained as follows
\begin{align}
\minop_{z \in \R^d} \Big\{ \norm{x-z}_p \,\big|\, \inner{w,z}+b=0\Big\}
\end{align}
Introducing $\delta=z-x$ we get
\begin{align}
\minop_{\delta \in \R^d} \Big\{ \norm{\delta}_p \,\big|\, \inner{w,\delta}+\inner{w,x}+b=0\Big\}
\end{align}
Note that by H{\"o}lder inequality one has $-\norm{w}_q\norm{\delta}_p \leq \inner{w,\delta} \leq \norm{w}_q\norm{\delta}_p$, which yields $\norm{\delta}_p \geq \frac{|\inner{w,\delta}|}{\norm{w}_p}$.
Noting that $\inner{w,\delta}=-(\inner{w,x}+b)$ we get $\norm{\delta}_p \geq \frac{|\inner{w,x}+b|}{\norm{w}_p}$ and equality is achieved when one has equality in the H{\"o}lder inequality.
\end{proof}
\fi

For the decision boundaries in $Q(x)$, with $c=\argmax_{r=1,\ldots,K} f^{(L+1)}_r(x)$, we define the quantity $d_D(x)$ as
\[\minop_{\stackrel{s=1,\ldots,K}{s\neq c}} \frac{\Big|\inner{V^{(L+1)}_c -V^{(L+1)}_s, x } + a^{(L+1)}_c - a^{(L+1)}_s \Big|}{\norm{V^{(L+1)}_c -V^{(L+1)}_s}_q}.\]
\begin{lemma}\label{le:bd-decision}
If $d_D(x)\leq d_B(x)$, then $d_D(x)$ is the minimal $l_p$-distance of $x$ to the decision boundary of the ReLU network $f^{(L+1)}$.
\end{lemma}

\begin{proof}
First we note that $d_D(x)$ is the distance of $x$ to the decision boundary for the linear multi-class classifier $g(x)=V^{(L+1)}x + a^{(L+1)}$
which is equal to $f^{(L+1)}$ on $Q(x)$.
If $d_D(x) \leq d_B(x)$ then the point realizing the minimal distance to the decision boundary $d_D(x)$ is inside $Q(x)$ and as $d_D(x)<d_B(x)$
there cannot exist another point  on the decision boundary of $f^{(L+1)}$ outside of $Q(x)$ having a smaller $l_p$-distance to $x$. Thus $d_D(x)$ is the minimal $l_p$-distance of $x$ to the decision boundary of $f^{(L+1)}$. 
\end{proof}

The next theorem combines both results to give a lower bound or the optimal solution to the optimization problem \eqref{eq:advopt}.
\begin{theorem}\label{th:main}
We get the following robustness guarantees:
\begin{enumerate}
\item If $d_B(x) \leq d_D(x)$, then $d_B(x)$ is a lower bound on the minimal $l_p$-norm of the perturbation necessary to change the class (optimal solution of \eqref{eq:advopt}).
\item If $d_D(x) \leq d_B(x)$, then $d_D(x)$ is equal to the minimal $l_p$-norm necessary to change the class (optimal solution of optimization problem \eqref{eq:advopt}).
\end{enumerate}
\end{theorem}
\begin{proof}
If $d_B(x)\leq d_D(x)$, then the ReLU classifier does not change on $B_p(x,d_B(x))$ and thus $d_B(x)$ is a lower bound on the minimal $l_p$-norm
perturbation necessary to change the class. The second statement follows directly by Lemma \ref{le:bd-decision}.
\end{proof}

\begin{figure}[t]\centering\includegraphics[scale=0.28]{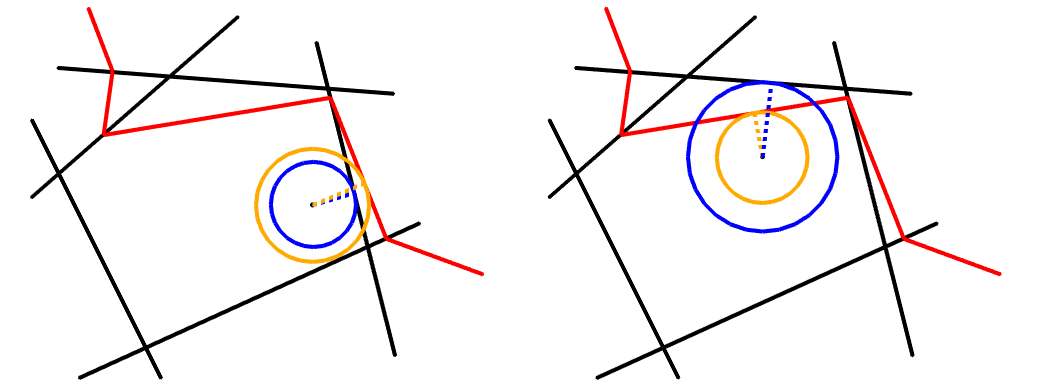}
	\caption{\textbf{Left}: the input $x$ is closer to the boundary of the polytope $Q(x)$ (black) than to the decision boundary (red). In this case the smallest perturbation that leads to a change of the decision lies outside the linear region $Q(x)$. \textbf{Right}: the input $x$  is closer to the decision boundary than to the boundary of $Q(x)$, so that the projection of the point onto the decision hyperplane provides the adversarial example with the smallest norm.}\label{fig:illustration}
\end{figure}

In Figure \ref{fig:illustration} we illustrate the different cases for $p=2$. On the left hand side $d_B(x)<d_D(x)$ and thus we get that  on the ball $B_2(x,d_B(x))$ the decision does not change, whereas in the rightmost plot we have $d_D(x)<d_B(x)$ and thus we obtain the minimum distance to the decision boundary.
Using Theorem \ref{th:main} we can provide robustness guarantees for every point and for some even compute the optimal robustness guarantees. Finally, we describe in the appendix (Section \ref{subsection:check_neighb_regions}) how one can improve the lower bounds by checking neighboring regions of $Q(x)$. Compared to the bounds of \cite{WonKol2018, WengEtAl2018} ours are slightly worse (see Table \ref{tab:our_bounds2} in the appendix). However, our bounds can be directly maximized and have a clear geometrical meaning and thus motivate directly
a regularization scheme for ReLU networks which we propose in the next section.

\section{Large margin and region boundary regularization}\label{sec:reg}

Using the results of Section \ref{sec:guar}, a classifier with guaranteed robustness requires large distance to the boundaries of the linear region as well
as to the decision boundary. Even the optimal guarantee (solution of \eqref{eq:advopt}) can be obtained in some cases. Unfortunately, as illustrated in Figures
\ref{fig:te_orig} and \ref{fig:margin_orig} for simple one hidden layer networks, the linear regions $Q(x)$ are small for normally-trained networks and thus no meaningful guarantees can be obtained. Thus we propose a new regularization scheme which simultaneously aims for each training point to achieve large distance to the boundary
of the linear region it lies in, as well as to the decision boundary. Using Theorem \ref{th:main} this directly leads to non-trivial robustness guarantees
of the resulting classifier. 

However, note that just maximizing the distance to the decision boundary might be misleading during training as this does not discriminate between points which are 
correctly (correct side of the decision hyperplane) or wrongly classified (wrong side of the decision hyperplane). Thus we introduce the signed version of $d_D(x)$, where $y$ is the true label of the point $x$, \begin{equation} \overline{d_D}(x) = \minop_{\stackrel{s=1,\ldots,K}{s\neq y}} \frac{f_y^{(L+1)}(x)-f_s^{(L+1)}(x)}{\norm{V^{(L+1)}_y -V^{(L+1)}_s}_q}. \end{equation} 
Please note that if $\overline{d_D}(x)\geq 0$, then $x$ is correctly classified, whereas $\overline{d_D}(x)<0$ if $x$
is wrongly classified. If $|\overline{d_D}(x)|\leq d_B(x)$, then it follows from Lemma \ref{le:bd-decision} that $|\overline{d_D}(x)|$ is the distance to the decision hyperplane.
If $|\overline{d_D}(x)|>d_B(x)$ this does not need to be any longer true, but $|\overline{d_D}(x)|$ is at least a good proxy as $\norm{V^{(L+1)}_y -V^{(L+1)}_s}_q$ is an estimate of the local
cross Lipschitz constant \cite{HeiAnd2017}. Finally, we propose to use the following regularization scheme:
\begin{definition}\label{def:MMR} Let $x$ be a point of the training set and define the Maximum Margin Regularizer ($MMR$) for $x$, for some $\gamma_B,\gamma_D\in\mathbb{R}_{++}$, as \begin{equation} MMR(x)= \max\Big(0,1-\frac{d_B(x)}{\gamma_B}\Big)+ \max\Big(0,1-\frac{\overline{d_D}(x)}{\gamma_D}\Big). \label{eq:regularizer}	\end{equation} \label{def:reg_single_point} \end{definition} 
The MMR penalizes distances to the boundary of the polytope if $d_B(x)\leq \gamma_B$ and positive distances ($x$ is correctly classified)
if $d_D(x)\leq \gamma_D$. Notice that wrongly classified points are always penalized. The part of the regularizer corresponding to $d_D(x)$ has been suggested in
\cite{ElsEtAl2018} in a similar form as a loss function for general neural networks without motivation from robustness guarantees. They have an additional
loss penalizing difference in the outputs with respect to changes at the hidden layers which is completely different from our geometrically motivated regularizer
penalizing the distance to the boundary of the linear region. The choice of $\gamma_B,\gamma_D$ allows different trade-off between the terms.
In particular $\gamma_D<\gamma_B$ (stronger maximization of $d_B(x)$) leads to more points for which the optimal robustness guarantee (case $d_D(x)\leq d_B(x)$) can be proved.
\\
For practical reasons, we also propose a variation of our MMR regularizer in \eqref{eq:regularizer}:
\begin{equation} \begin{split} kMMR(x)= &\frac{1}{k_B}\sum_{i=1}^{k_B}\max\Big(0,1-\frac{d_B^{i}(x)}{\gamma_B}\Big)\\ &+ \frac{1}{k_D}\sum_{i=1}^{k_D}\max\Big(0,1-\frac{\overline{d_D^{i}}(x)}{\gamma_D}\Big), \label{eq:regularizer_bottom_k} \end{split} \end{equation}
where $d_B^{i}(x)$ is the distance of $x$ to the $i$-th closest hyperplane of the boundary of the polytope and $d_D^{i}(x)$ is the analogue for the decision boundaries. Basically, we are optimizing, instead of the closest decision hyperplane, the $k_D$-closest ones and analogously the $k_B$-closest hyperplanes
defining the linear region $Q(x)$ of $x$. This speeds up the training time as more hyperplanes are moved in each update. Moreover, when deriving lower bounds using more than one linear
region, one needs to consider more than just the closest boundary hyperplane. Finally, many state-of-the-art schemes for proving lower bounds \cite{TjeTed2017,WonKol2018,WengEtAl2018} work well only if 
the activation status of most neurons is constant for small changes of the points. This basically amounts to ensure that all the hyperplanes are sufficiently far away, which is exactly what our regularization scheme is aiming at. Thus our regularization scheme also helps to achieve state-of-the-art provable robustness with other certification methods (\cite{TjeTed2017,WonKol2018}). This is also the reason why the term pushing the polytope boundaries away is essential. Just penalizing the distance to the decision boundary is
not sufficient to prove good lower bounds on the minimal distance to the decision boundary as we will show in our experiments (Table \ref{tab:MMR2}). Compared to the regularization scheme in \cite{WonKol2018} using a dual feasible point
of the robust loss, our approach has a direct geometric interpretation and allows to derive the exact minimal perturbation for some fraction of the test points varying from dataset
to dataset but it can be as high as $99\%$. In practice, we gradually decrease $k_B$ and $k_D$ in \eqref{eq:regularizer_bottom_k} during training
so that only the closest hyperplanes of each training point influence the regularizer.

Thus, denoting the cross entropy loss $CE(f(x),y)$, the final objective of our models is
\begin{equation} \frac{1}{n}\sum_{i=1}^{n}CE(f(x_i),y_i) +\lambda\, kMMR(x_i), \label{eq:obj}\end{equation} 
where $(x_i,y_i)_{i=1}^n$ is the training data and $\lambda\in\mathbb{R}_+$ the regularization parameter. 
Figure \ref{fig:plot_margin} shows the effect of the regularizer. Compared to the unregularized case the size of the linear regions is significantly increased.
\begin{figure}[t!] \centering \subfloat[]{\includegraphics[height=0.2\textwidth]{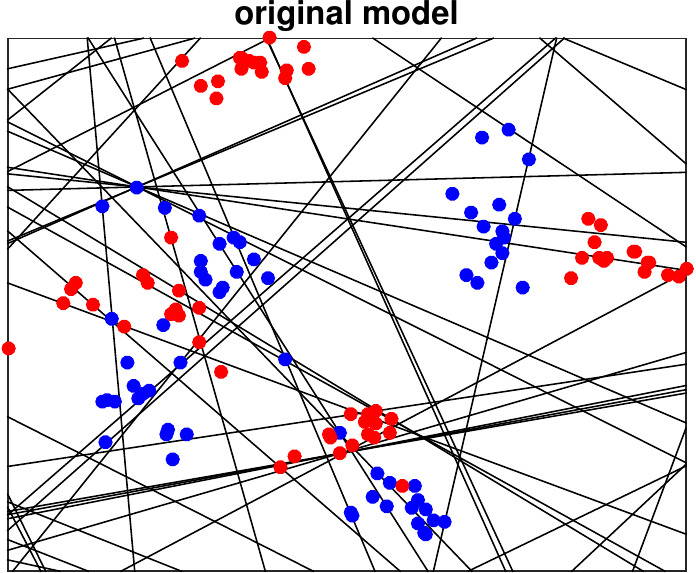}\label{fig:te_orig}}
	\subfloat[]{\includegraphics[height=0.2\textwidth]{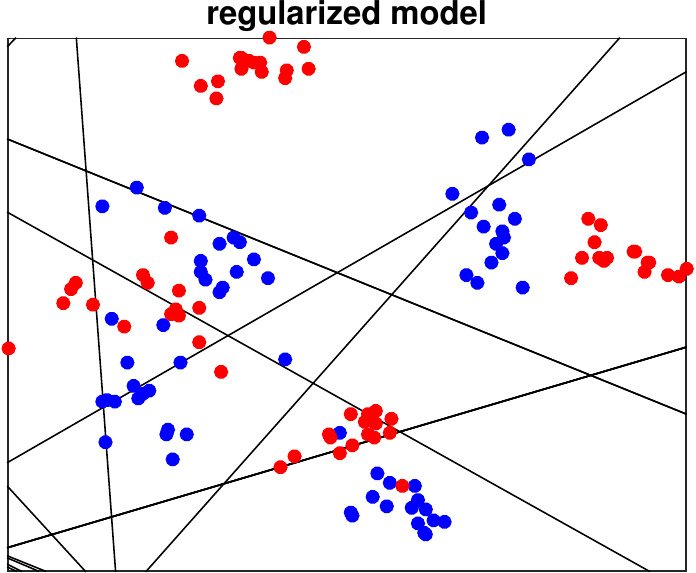}\label{fig:te_reg}}
	\vskip\baselineskip
	\vspace{-4mm}
	\subfloat[]{\includegraphics[height=0.2\textwidth]{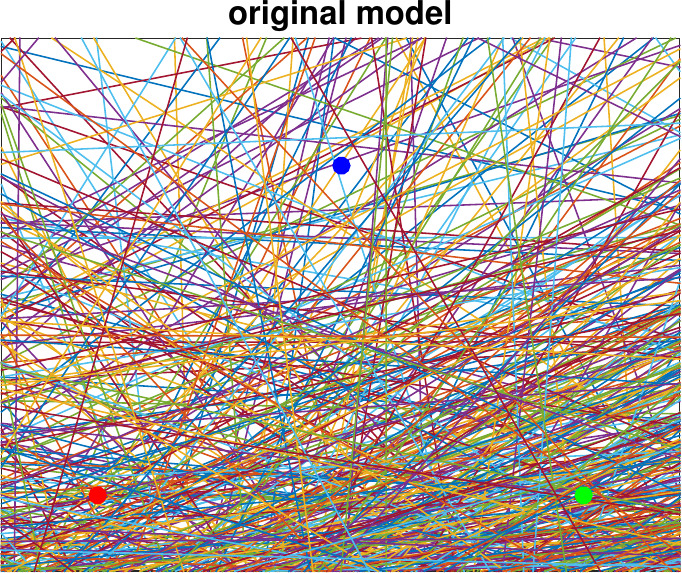}\label{fig:margin_orig}}
	\subfloat[]{\includegraphics[height=0.2\textwidth]{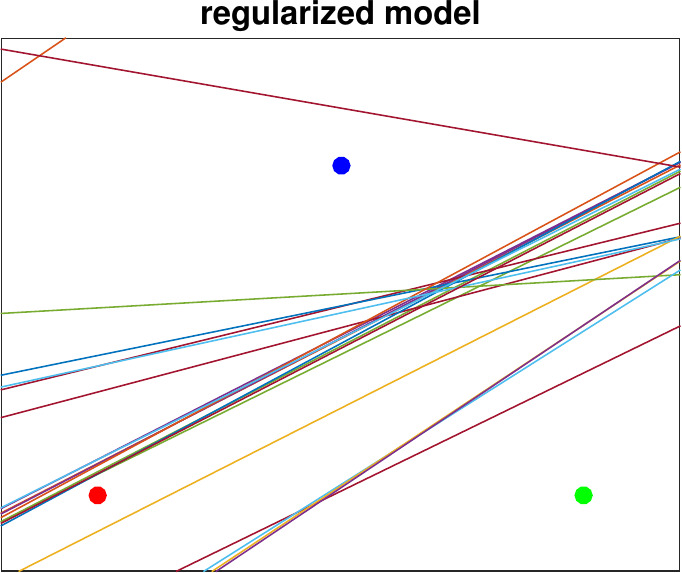}\label{fig:margin_reg}}
	\caption{\textbf{The effects of MMR}. \textit{Top row}: we train two networks with one hidden layer, 100 units, on 128 points belonging to two classes (red and blue). Figure \ref{fig:te_orig} shows the points and how the input space is divided in regions on which the classifier is linear. Figure \ref{fig:te_reg} is the analogue for our MMR regularized model. \textit{Bottom row}: we show region boundaries (one hidden layer, 1024 units)
on a 2D slice of $\R^{784}$ spanned by three random points from different classes of the MNIST training set. We observe a clear maximization of the linear regions for the MMR-regularized case (Figure \ref{fig:margin_reg}) versus the non-regularized case (Figure \ref{fig:margin_orig}).}\label{fig:plot_margin}
\end{figure}

\section{Experiments}\label{sec:exp}
\begin{table*}[t]
	\centering
	\caption{\textbf{Comparison of different methods regarding robustness.} We here report the statistics of 6 training schemes: \textit{plain} (usual training), \textit{at} (adversarial training \cite{MadEtAl2018}), KW \cite{WonKol2018}, Xiao et al \cite{XiaoEtAl18},  MMR (ours), and MMR+\textit{at} (MMR plus adversarial training). We show, in percentage, test error (TE), lower (LB) and upper (UB) bounds on the robust test error at the threshold $\epsilon$ indicated for each dataset. The robustness statistics are computed on the first 1000 points of the respective test sets for $l_\infty$-robustness, on the full test set for $l_2$-robustness. 
	KW models marked with * are taken from \cite{WonKol2018,WonEtAl18}, while the other are retrained according to the available code (but the bounds are again evaluated with the discussed combination of multiple techniques). 
	The results and bounds of Xiao et al are taken directly as reported in \cite{XiaoEtAl18}.
	The model indicated by $^1$ has been obtained at epoch 50 instead of 100 due to optimization issues.
	}
	\begin{tabular}{l || r r r | r r r || r r r | r r r }
		\multicolumn{1}{l}{} & \multicolumn{6}{c}{\textbf{$l_\infty$-norm robustness}} & \multicolumn{6}{c}{\textbf{$l_2$-norm robustness}}\\ [0.5mm]
		\textit{training}& \multicolumn{3}{c|}{FC1} & \multicolumn{3}{c||}{CNN} & \multicolumn{3}{c|}{FC1} & \multicolumn{3}{c}{CNN} \\
		\cline{2-13}
		\textit{scheme}& TE & LB & UB & TE & LB & UB & TE & LB & UB & TE & LB & UB \\
		\hline
		\multicolumn{13}{l}{}\\
		\multicolumn{1}{l}{\textbf{MNIST}}& \multicolumn{6}{c||}{$\epsilon=0.1$} & \multicolumn{6}{c}{$\epsilon=0.3$}\\
		\hline
		plain & 1.44 & 93.0  & 100 & 0.91  & 74.0 & 100  &1.73 & 9.7 & 66.3& 0.85 & 3.1 & 100 \\
		at & 0.92 & 10.0  & 99.0 & 0.82 & 3.0 & 100  & 1.15 &  2.6 & 16.9 & 0.87 & 1.8 & 100 \\
		KW & 3.19 & 10.9 & 10.9 & 1.26* & 4.4  & 4.4 &1.19 & 2.4 & 5.2 & 1.11 & 2.2  & 6.0 \\
		Xiao et al & - & - & - & 1.32 & 4.9 & 5.7 & - & - & - & - & - & - \\
		\cdashline{1-13}
		MMR & 2.11 & 22.5 & 24.9 & 1.65 &  6.0 &  6.0  & 2.40  & 5.9 & 8.8 & 2.57 & 5.8 &  11.6 \\
		MMR+at & 2.04 &  14.0 &  14.1 & 1.19 &3.6 & 3.6  & 1.77 & 3.8 & 6.4 & $2.12^1$ & 4.6 & 9.7\\
		\hline
		
		\multicolumn{13}{l}{}\\
		\multicolumn{1}{l}{\textbf{F-MNIST}}&\multicolumn{6}{c||}{$\epsilon=0.1$} & \multicolumn{6}{c}{$\epsilon=0.3$}\\
		\hline
		plain & 9.49 & 100  & 100 & 9.50 & 96.0  & 100 & 9.70 & 42.8 & 91.8 & 9.32 & 57.1  & 100  \\
		at & 11.69 & 29.5  & 95.5 & 11.54 & 21.5 & 73.0  & 9.15 & 19.9 & 61.4 & 8.10 & 20.4 & 100 \\
		KW & 21.31 & 32.8 & 32.8 & 21.73* & 32.4 & 32.4  & 11.24 & 17.2 & 22.7 & 13.08 & 18.5 & 21.7 \\
		\cdashline{1-13}
		MMR & 18.11 & 37.6   & 42.0 & 14.52 & 33.2 & 33.6 & 13.28  & 25.0 & 28.0 & 12.85 & 25.4 & 35.3 \\
		MMR+at & 15.84  & 31.3 & 34.8 & 14.50 & 26.6 & 30.7 & 12.12 & 19.7 & 23.4 & 13.42 & 26.2 &39.1 \\
		\hline
		
		\multicolumn{13}{l}{}\\
		\multicolumn{1}{l}{\textbf{GTS}}&\multicolumn{6}{c||}{$\epsilon=\nicefrac{4}{255}$} & \multicolumn{6}{c}{$\epsilon=0.2$}\\
		\hline
		plain & 12.72 & 61.0 & 77.0 & 7.11  & 63.0 & 99.5 & 12.24  &37.5 & 44.7 & 6.78 & 33.3 & 99.2 \\
		at & 9.33 & 34.5  & 48.5 &  6.84 & 29.5& 81.0 & 13.55 & 33.1 & 43.2 & 8.75 & 23.8 & 98.6 \\
		KW & 13.99 & 33.0 & 33.0 & 15.56 & 36.1 & 36.6 & 16.84 & 16.9 & 31.2 & 14.33  & 28.7  & 34.7 \\ 
		\cdashline{1-13}
		MMR & 14.29 & 39.8 & 39.8 & 13.31 & 49.5 & 49.6  & 14.55 & 33.2 & 34.7 & 14.22 & 36.2 & 36.9 \\
		MMR+at & 13.10 & 33.1 & 35.4 & 14.88 & 38.3 &  38.4 & 13.94 & 29.7 & 32.1 & 15.35 & 32.1 & 33.2 \\
		\hline
		
		\multicolumn{13}{l}{}\\
		\multicolumn{1}{l}{\textbf{CIFAR-10}}&\multicolumn{6}{c||}{$\epsilon=\nicefrac{2}{255}$} & \multicolumn{6}{c}{$\epsilon=0.1$}\\
		\hline
		plain & - & - & - & 24.63 & 91.0 & 100 & - & - & - & 23.31 & 47.2 & 100 \\
		at & - & - & - & 27.04 & 52.5 & 88.5  & - & - & - & 25.82 & 35.8 & 100 \\
		KW & - & - & - & 38.91* & 46.6 & 48.0 & - & - & - & 40.24 & 43.9 & 49.0 \\
		Xiao et al & - & - & - & 38.88 & 50.1 & 54.1 & - & - & - & - & - & - \\
		\cdashline{1-13}
		MMR & - & - & - & 34.61 & 57.5 & 61.0 & - & - & - & 40.92 & 50.6 & 57.1 \\
		MMR+at & - & - & - & 35.38 & 47.9 & 54.2 & - & - & - & 37.75 & 43.9 & 53.3\\
		\hline
		
		\hline
	\end{tabular}
	\label{tab:main_exps4_2}
\end{table*}

\begin{figure}[t]
	\centering
	\includegraphics[width=0.49\columnwidth]{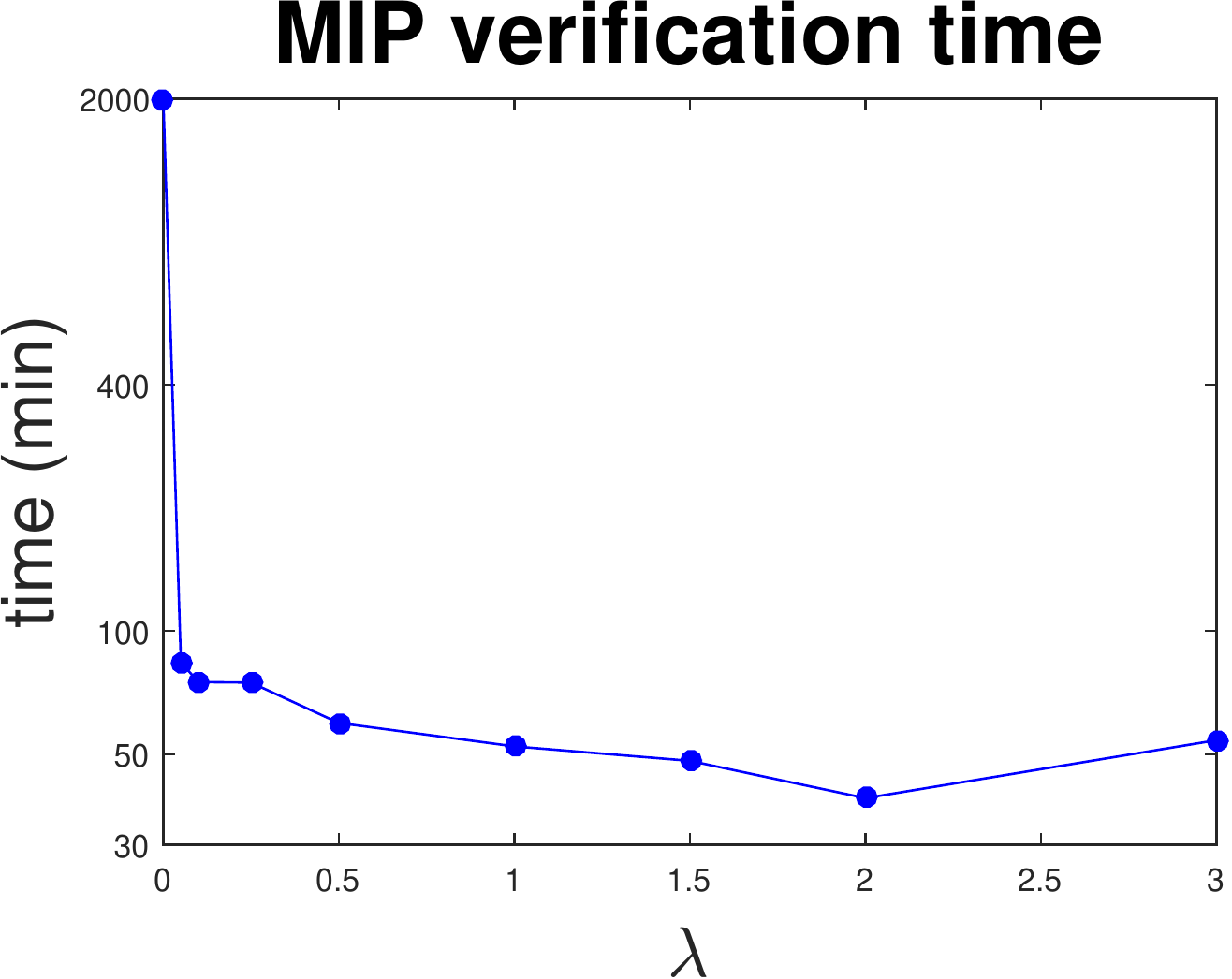}
	\includegraphics[width=0.49\columnwidth]{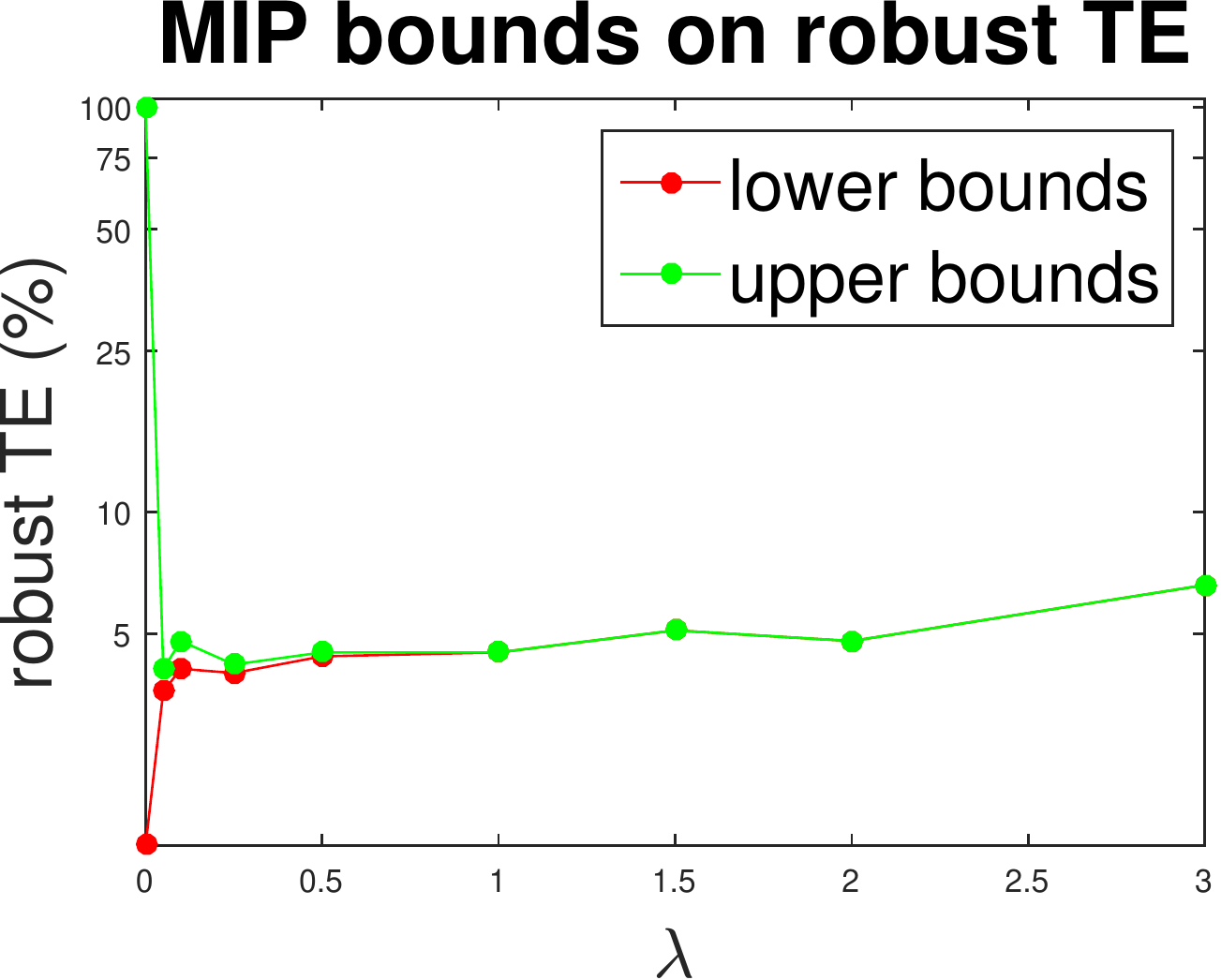}
	\caption{\textbf{Verifiability of models.} We show the runtime (left) in minutes that MIP \cite{TjeTed2017} takes to verify 1000 points, setting a timeout of 120s (that is the mixed-integer optimization stops anyway after the time limit is reached), with models trained with different values of $\lambda$ (see Equation \eqref{eq:obj}). Note that a logarithmic scale is used on the $y$-axis. Moreover, we report (right) lower (red) and upper (green) bounds on robust test error. The \textit{plain} model, trained without MMR ($\lambda=0$) needs 51 times more to be verified, with only 1\% of the points certified. Conversely, even with a light MMR regularization lower and upper bounds are tight.}\label{fig:MIP_stats}
\end{figure}

\subsection{Main experiments}
We provide a variety of experiments aiming to show the state-of-the-art performance of MMR to achieve provably robust classifiers. We make our code and models publicly available\footnote{\url{https://github.com/max-andr/provable-robustness-max-linear-regions}}.
We use four datasets: MNIST, German Traffic Signs (GTS) \cite{GTSB2012}, Fashion MNIST \cite{XiaoEtAl2017} and CIFAR-10. We consider in the paper robustness wrt to both $l_2$ and $l_\infty$ distances. We use two criteria.
First, upper and lower bounds on the robust test error for a given threshold $\epsilon$, that is the largest achievable test error if every test input can be modified with a perturbation $\delta$ such that $\norm{\delta}_p\leq\epsilon$ and $\delta$ is chosen so that $x+\delta$ is classified wrongly. Second, we show in Section \ref{sec:further_exp} and in the appendix results for lower 
and upper bounds on the minimal $\norm{\delta}_p$ from \eqref{eq:advopt}.

Lower bounds on the robust test error for $l_2$ and $l_\infty$ are computed using the attack via Projected Gradient Descent (PGD) \cite{MadEtAl2018}. Upper bounds on the robust
test error are computed using the approach of \cite{WonKol2018}. Additionally, \cite{TjeTed2017} yields upper and lower bounds on the robust test error via mixed integer programming (MIP). We use the solver for the MIP \cite{TjeTed2017} with a timeout of 120s per point, that is if the point has not been verified in this time the solver is stopped. 
This technique is currently effective for $l_\infty$ but not for $l_2$, where basically in almost every case the timeout is reached and thus we discard for $l_2$ the MIP evaluation.
Finally, we combine the upper bounds on the robust test error found by \cite{WonKol2018} and \cite{TjeTed2017} by counting the fraction of points that can be certified with at least one of the two methods. We compare our own guarantees from Theorem \ref{th:main} with the ones of \cite{WonKol2018} in the appendix. Moreover, we also combine the lower bounds on the robust test error found by PGD, the MIP and the attack of \cite{CroHei18,CroRauHei2019}.

We compare six methods: plain training, adversarial training of \cite{MadEtAl2018} which has been shown to significantly increase robustness, the robust loss of \cite{WonKol2018} which supports both $l_2$ and $l_\infty$ norms (denoted as KW), the training scheme of \cite{XiaoEtAl18} (denoted as Xiao et al), our regularization scheme MMR and MMR together with adversarial training again as in \cite{MadEtAl2018}.
All schemes are evaluated on a one hidden layer fully connected network with 1024 hidden units (FC1) and a convolutional network (CNN) with 2 convolutional and 2 dense layers as used in \cite{WonKol2018}. For more details see Appendix\ifpaper\ \ref{app:main_exp}\fi. Note that since code for \cite{XiaoEtAl18} is not available, we just show the $l_\infty$ results from their paper (evaluated on full test sets with PGD lower bounds, and MIP upper bounds), which are available only for MNIST and CIFAR-10 for the same CNN architecture.

\textbf{Improvement of robustness:} 
The results can be found in Table \ref{tab:main_exps4_2}. For CIFAR-10 we show only CNN models as fully connected networks do not have good test performance.
We report clean test error, lower and upper bounds on robust test error at the threshold indicated in Table \ref{tab:main_exps4_2}, computed on 1000 points for $l_\infty$, and on the full test sets for $l_2$ (as we do not do the MIP evaluation there).\\
Both KW and MMR+at achieve similar performance regarding lower and upper bounds on the robust test error. For $l_\infty$ our MMR+at achieves overall better performance than KW, 
sometimes with significantly better clean test error like on F-MNIST. In some cases, e.g. on CIFAR-10, MMR+at provides slightly worse upper bounds, but instead preserves better test error. \\
For $l_2$ KW performs better than MMR+at with the exception of GTS. 
This is to be expected as we use for the evaluation of the upper bounds on the robust test error the approach of \cite{WonKol2018} which are directly optimized by their robust training
procedure. 
Note that both KW and MMR/MMR+at outperform by large margin plain and adversarial training regarding provable robustness (upper bounds on the robust test error). With the exception of the CNN-$l_2$ models
for MNIST and F-MNIST, the upper bound on robust test error of MMR+at is smaller than the lower bound of the plain model. Thus our models are provably better than the plain models regarding robust test error.  Moreover, regarding robustness wrt to $l_\infty$ the gaps between lower and upper bounds are often very small for KW and MMR/MMR+at showing that both techniques lead to models which are easier to check via the MIP which we discuss next in more detail.

\textbf{Enhancing verifiability:} A key aspect of any robust training should be the ability of producing models both resistant to adversarial manipulation and being \textit{verifiable}, in the sense of having guarantees on the minimal perturbation changing the decision for a test input. In fact, even if empirically model seems to be robust, only computing certificates allows to completely trust it. In our experiments, we could use successfully the method of \cite{WonKol2018} to get meaningful guarantees, which so far, as noted in \cite{RagSteLia2018}, could be achieved only on models provided by the specific training of \cite{WonKol2018}. Moreover, the MIP is too slow to run on standard models, so that ad hoc techniques have been developed to train classifiers verifiable by MIP \cite{XiaoEtAl18}. Our MMR training produces also models which can be checked by MIP. This is due to the fact that the hyperplanes representing the boundary of the polytope are actually the boundary between the different regimes (identity or zero function) of ReLU units, so that pushing the hyperplanes implies that many inputs of ReLU units have constant sign in a wide region around the data points (and having ReLU units with unstable signs is the main slowdown factor for the MIP solver).\\
In this context, \cite{XiaoEtAl18} develop a specific technique aiming at inducing stability of ReLU signs.
Therefore, we provide a comparison of MMR to the ReLU-stability loss of \cite{XiaoEtAl18} in Table \ref{tab:main_exps4_2}. On MNIST our MMR+at model of the same CNN architecture has much tighter upper bounds on the robust error -- 3.6\% compared to 5.7\% of Xiao et al. Moreover, we have no gap between the lower and upper bounds on the robust test error, and our upper bounds are \textit{lower} than the lower bounds of Xiao et al. This suggests that our training scheme is better at both improving robustness and enhancing verifiability than the approach of Xiao et al \cite{XiaoEtAl18}. Additionally, on CIFAR-10 we have similar upper bounds, but better test error and better lower bounds. \\
To illustrate the speedup in verification time for MMR models, we show in Figure \ref{fig:MIP_stats} the time MIP needs to run on 1000 points (with timeout of 120s per point) and lower and upper bounds on robust test error wrt $l_\infty$-distance at $\epsilon=0.1$. We use CNNs trained on MNIST with $\gamma_B=\gamma_D=0.15$ and different values of $\lambda$ representing the weight of our regularizer in the loss \eqref{eq:obj} (for $\lambda=0$ one gets the plain model). It is clear that MIP performs poorly both in runtime (almost 2000 minutes) and performance (LB=1\%, UB=100\%) on the plain model. In contrast, the MMR models are verified quickly (between 35 and 79 minutes) and almost completely (the rate of certified points is between 99.3\% and 100\%), that is lower and upper bounds are close or even equal. Please note that both statistics improve with increasing $\lambda$ up to 2, while runtime gets worse with a larger value (as at $\lambda=3$ the classifier becomes less robust and then requires a higher computational effort to be certified).

\begin{table*}[t]
	\centering
	\caption{\textbf{Full version of MMR is necessary.} We compare the statistics of models trained with the full version of MMR as in \eqref{eq:regularizer} and \eqref{eq:regularizer_bottom_k} (left) and with only the second part penalizing the distance to the decision boundary (right). While the test error is for the full test set, the lower and upper bounds on the $l_2$-norm of the optimal adversarial manipulation are compared on the first 1000 points of the test set. One can clearly see that the lower bounds improve significantly when one uses the full MMR regularization.}
	\begin{tabular}{c l || c c c| c c c }
		\multicolumn{2}{c}{} &\multicolumn{3}{c|}{MMR-\textit{full}}& \multicolumn{3}{c}{MMR-$d_D$}\\
		dataset & model & test error & lower bounds & upper bounds & test error & lower bounds & upper bounds \\
		\hline
		\multirow{2}{*}{MNIST}& FC1 & 1.51\% & \textbf{0.69} & \textbf{1.69} & 0.93\% & 0.35 & \textbf{1.69} \\
		& FC10& 1.87\% & \textbf{0.48} & 1.48 & 1.21\% & 0.20 & \textbf{1.62} \\
		\hline
		\multirow{2}{*}{GTS}& FC1 & 11.15\% & \textbf{0.69} & \textbf{0.69} & 12.09\% & 0.48 & 0.63 \\
		& FC10& 12.82\% & \textbf{0.64} & \textbf{0.67} & 12.41\% & 0.12 & 0.48 \\		
		\hline
		\multirow{2}{*}{F-MNIST}& FC1 & 10.22\% & \textbf{0.50} & 0.85 & 9.83\% & 0.31 & \textbf{1.30} \\
		& FC10&11.73\% & \textbf{0.68} & \textbf{1.18} & 10.32\% & 0.13 & 1.15 
	\end{tabular}
	\label{tab:MMR2}
\end{table*}

\begin{table*}[t]
	\centering
		\caption{\textbf{Occurrence of guaranteed optimal solutions.} For each dataset and architecture we report the percentage of points of the test set for which we can compute the guaranteed optimal solution of \eqref{eq:advopt} for models trained without (plain setting) and with MMR regularization. We show the test error of the models as well. In most of the fully connected cases, we achieve certified optimal solutions for a significant fraction of the points without degrading significantly, or sometimes improving, the test error. Moreover, where we have a meaningful number of points with provable minimal perturbation, it is interesting to check how much worse the lower bounds computed by \cite{WonKol2018} (\textit{LB}) are. Thus the column \textit{opt vs LB} indicates, in percentage, how much larger the $l_2$-norm of the optimal solution of \eqref{eq:advopt} is compared to its lower bound, explicitly $\Big(\nicefrac{\norm{\delta_{opt}}_2}{LB}-1\Big)\times 100\%$.}
	\begin{tabular}{c l || c c c| c c }
		\multicolumn{2}{c}{} &\multicolumn{3}{c|}{MMR training}& \multicolumn{2}{c}{plain training}\\
		dataset & model & test error & optimal points & opt vs LB &  test error & optimal points \\
		\hline
		\multirow{2}{*}{MNIST}& FC1 & 1.51\% & 0.20\% & 14.11\% & 1.59\% & 0.02\% \\
		& FC10& 1.87\% & 0.06\% & - & 1.81\% & 0.02\% \\
		& CNN& 1.17\% & 0.00\% & - & 0.97\% &0.00\%   \\
		\hline
		\multirow{2}{*}{GTS}& FC1 & 11.15\% & 99.97\% & 0.59\% & 12.27\% & 1.12\%  \\
		& FC10& 12.82\% & 94.86\% & 0.66\% & 11.28\% & 0.14\% \\
		& CNN&7.40\% & 0.00\% & - & 6.73\% &0.00\%  \\	
		\hline
		\multirow{2}{*}{F-MNIST}& FC1 & 10.22\% & 11.22\% & 6.53\% & 9.61\% & 0.37\% \\
		& FC10& 11.73\% & 9.90\% & 7.27\% & 10.53\% & 0.04\% \\
		& CNN& 10.30\% & 0.09\% & - & 8.86\% &0.00\%  
	\end{tabular}
	\label{tab:opt_solutions}
\end{table*}

\subsection{Further experiments}\label{sec:further_exp}
We present a series of experiments for a detailed understanding of how MMR works. For this section we consider models trained to be $l_2$-robust and evaluate robustness as the average $l_2$-norm of the perturbations necessary to change the class. Lower bounds on this perturbation are computed either with \cite{WonKol2018} or our method (Theorem \ref{th:main}), while upper bounds are provided by Carlini-Wagner $l_2$-attack (CW) \cite{CarWag2016}. We also introduce a second fully connected architecture, FC10, with 10 hidden layers (see Appendix for details).

\textbf{Importance of linear regions maximization:}
In order to highlight the importance of both parts of the MMR regularization, i) penalization of the distance to decision boundary and ii) penalization of the distance to boundary of the polytope, we train, for each dataset/architecture, models penalizing only the distance to the decision boundary, that is the second term in the r.h.s. of \eqref{eq:regularizer} and \eqref{eq:regularizer_bottom_k}. We call this partial regularizer MMR-$d_D$, in contrast to the full version MMR-\textit{full}. Then we compare the lower and upper bounds on the solution of \eqref{eq:advopt} for MMR-$d_D$ and MMR-\textit{full} models. For a fair comparison we consider models with similar test error. We clearly see in Table \ref{tab:MMR2} that the lower bounds, computed by the method presented in \cite{WonKol2018}, are always significantly better when MMR-\textit{full} is used, while the behavior of the upper bounds does not clearly favor one of the two. This result shows that in order to
get good lower bounds one has to increase the distance of the points to the boundaries of the polytope.

\textbf{Guaranteed optimal solutions via MMR:} Theorem \ref{th:main} provides a simple and efficient way to obtain in certain cases the solution of \eqref{eq:advopt}. Although for normally trained networks the conditions are rarely satisfied, we show in Table \ref{tab:opt_solutions} that for the MMR-models for fully connected networks for a significant fraction of the test set we obtain
the globally optimal solution of \eqref{eq:advopt}, that is the true $l_2$-robustness. Moreover, we report how much better our globally optimal solutions are compared to the lower
bounds of \cite{WonKol2018}. Interestingly, we can provably get the true robustness for around $10\%$ of points for F-MNIST and for over $98\%$ of the points on GTS for the 
case of fully connected networks. For these cases the optimal solutions have roughly $7\%$ larger $l_2$-norm for F-MNIST and $0.5\%$ larger for  GTS than the lower bounds. 
Note that currently the MIP \cite{TjeTed2017} is not efficient for $l_2$ even though there is room for improvement.
Globally optimal solutions for larger networks achieved via our method can serve as a test both for lower and upper bounds. This is an important issue as currently large parts of the community relies just on upper bounds of robustness using attack schemes like the CW-attack which we address in the the next paragraph.

\ifpaper
\begin{figure}[t]\centering\includegraphics[scale=0.5]{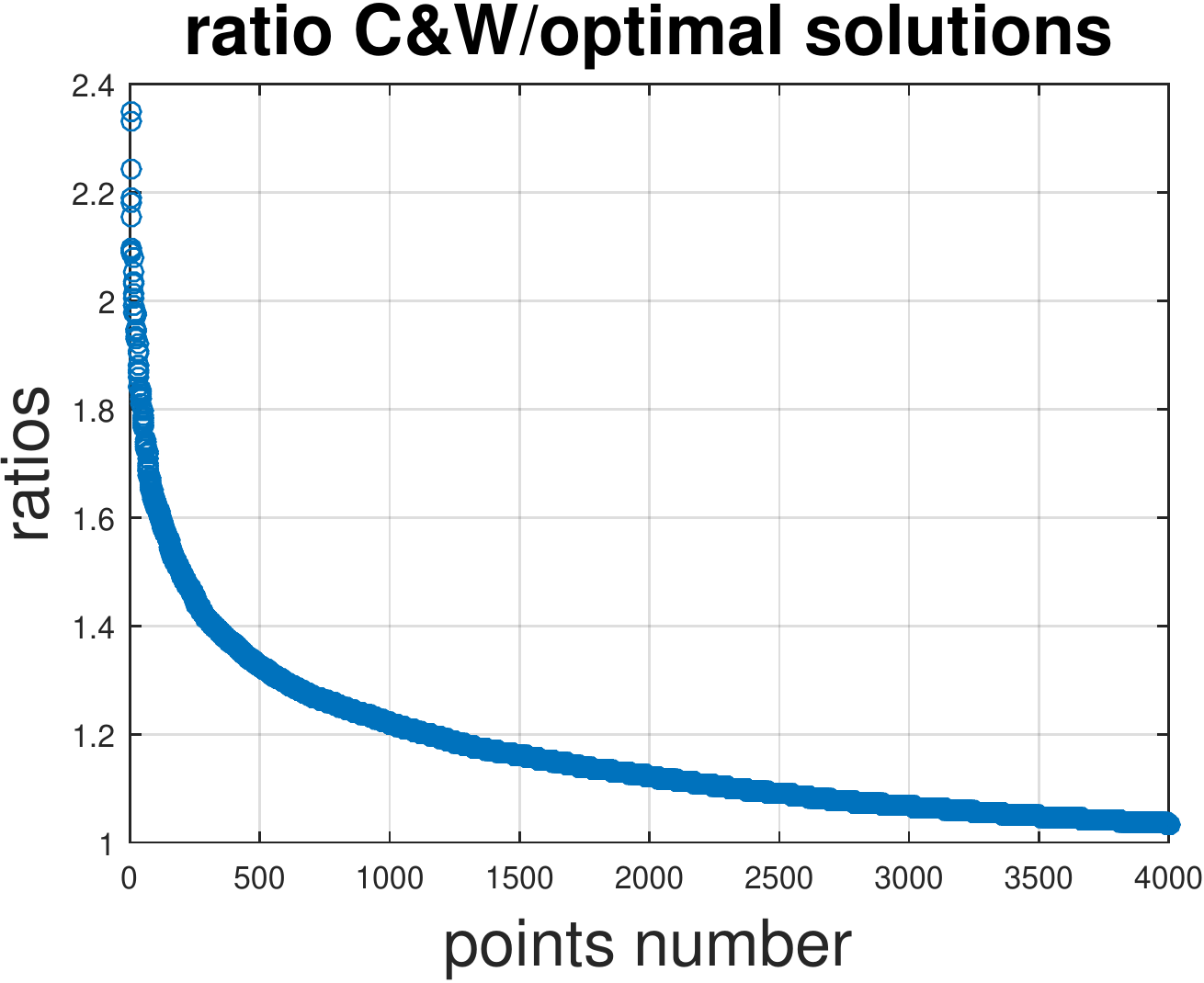}\caption{We report the descending sorted ratios $\norm{\delta_{CW}}_2$/$\norm{\delta_{opt}}_2$ (norm of the outcomes of CW-attack divided by the norm of minimal adversarial examples) with regard to a FC1 model on GTS dataset trained with our regularizer.} \label{fig:CW_long}
\end{figure}
\fi

\textbf{Evaluation of CW-attack:} The CW-attack \cite{CarWag2016} which we use for our upper bounds in Tables \ref{tab:MMR2} and \ref{tab:main_exps4_v1_2} is considered state-of-the-art. Thus it is interesting to see how close it is to the globally optimal solution. On GTS FC1 we find for the MMR model with best test error from Table \ref{tab:main_exps4_v1_2} the globally optimal solution of \eqref{eq:advopt} for 12596 out of 12630 test points. We compare on this subset the optimal norm $\norm{\delta_{opt}}_2$ to $\norm{\delta_{CW}}_2$ obtained by the \ifpaper CW-attack and plot the ratios $\norm{\delta_{CW}}_2$/$\norm{\delta_{opt}}_2$ in descending order in Figure \ref{fig:CW_long} (note that we truncate at 4000 points). \else CW-attack.\fi While the CW-attack performs in general well, there are 2330 points ($18.5\%$ of the test set) where the CW-attack has at least 10\% larger norms and 1145 points (around $9.1\%$) with at least 20\% larger norms. The maximal relative difference is 235\%. Thus at least on a pointwise
basis evaluating robustness with respect to an attack scheme can significantly overestimate robustness, see also \cite{CroHei18}. This shows the importance of techniques to prove lower bounds.
Moreover, the time to compute the adversarial examples by the CW-attack is 16327s, while our technique provides both lower bounds and optimal solutions in 1701s.

\section{Conclusion}
We have introduced a geometrically motivated regularization scheme which leads to models which are provably robust according to the current state-of-the-art certification methods of \cite{WonKol2018, TjeTed2017}. In particular, it performs as well as the state-of-the-art models of \cite{WonKol2018} in terms of certified robust test error, and better or similarly to \cite{XiaoEtAl18}. Finally, our method based on the linear regions allows to obtain the globally minimal adversarial perturbation in a significant fraction of cases for large fully connected networks which can be used to test lower and upper bounds.

\section*{Acknowledgements}
We would like to thank Vincent Tjeng for helping to set up the MIP evaluation. Furthermore, we thank Eric Wong and Zico Kolter for adapting their code to the $l_2$-norm, as well as for many helpful discussions.

\bibliographystyle{abbrv}

\begin{thebibliography}{10}

\bibitem{AroEtAl2018}
R.~Arora, A.~Basuy, P.~Mianjyz, and A.~Mukherjee.
\newblock Understanding deep neural networks with rectified linear unit.
\newblock In {\em ICLR}, 2018.

\bibitem{AthEtAl2018}
A.~Athalye, N.~Carlini, and D.~Wagner.
\newblock Obfuscated gradients give a false sense of security: Circumventing
  defenses to adversarial examples.
\newblock {\em arXiv preprint arXiv:1802.00420}, 2018.

\bibitem{BasEtAl2016}
O.~Bastani, Y.~Ioannou, L.~Lampropoulos, D.~Vytiniotis, A.~Nori, and
  A.~Criminisi.
\newblock Measuring neural net robustness with constraints.
\newblock In {\em NIPS}, 2016.

\bibitem{CarEtAl2017}
N.~Carlini, G.~Katz, C.~Barrett, and D.~L. Dill.
\newblock Provably minimally-distorted adversarial examples.
\newblock preprint, arXiv:1709.10207v2, 2017.

\bibitem{CarWag2017}
N.~Carlini and D.~Wagner.
\newblock Adversarial examples are not easily detected: Bypassing ten detection
  methods.
\newblock In {\em ACM Workshop on Artificial Intelligence and Security}, 2017.

\bibitem{CarWag2016}
N.~Carlini and D.~Wagner.
\newblock Towards evaluating the robustness of neural networks.
\newblock In {\em IEEE Symposium on Security and Privacy}, 2017.

\bibitem{CisEtAl2017}
M.~Cisse, P.~Bojanowksi, E.~Grave, Y.~Dauphin, and N.~Usunier.
\newblock Parseval networks: Improving robustness to adversarial examples.
\newblock In {\em ICML}, 2017.

\bibitem{CroHei18}
F.~Croce and M.~Hein.
\newblock A randomized gradient-free attack on relu networks.
\newblock In {\em GCPR}, 2018.

\bibitem{CroRauHei2019}
F.~Croce, J.~Rauber, and M.~Hein.
\newblock Scaling up the randomized gradient-free adversarial attack reveals
  overestimation of robustness using established attacks.
\newblock In preparation, 2019.

\bibitem{ElsEtAl2018}
G.~F. Elsayed, D.~Krishnan, H.~Mobahi, K.~Regan, and S.~Bengio.
\newblock Large margin deep networks for classification.
\newblock preprint, arXiv:1803.05598v1, 2018.

\bibitem{GooShlSze2015}
I.~J. Goodfellow, J.~Shlens, and C.~Szegedy.
\newblock Explaining and harnessing adversarial examples.
\newblock In {\em ICLR}, 2015.

\bibitem{GuRig2015}
S.~Gu and L.~Rigazio.
\newblock Towards deep neural network architectures robust to adversarial
  examples.
\newblock In {\em ICLR Workshop}, 2015.

\bibitem{HeiAnd2017}
M.~Hein and M.~Andriushchenko.
\newblock Formal guarantees on the robustness of a classifier against
  adversarial manipulation.
\newblock In {\em NIPS}, 2017.

\bibitem{HuaEtAl2016}
R.~Huang, B.~Xu, D.~Schuurmans, and C.~Szepesvari.
\newblock Learning with a strong adversary.
\newblock In {\em ICLR}, 2016.

\bibitem{KatzEtAl2017}
G.~Katz, C.~Barrett, D.~Dill, K.~Julian, and M.~Kochenderfer.
\newblock Reluplex: An efficient smt solver for verifying deep neural networks.
\newblock In {\em CAV}, 2017.

\bibitem{KinEtAl2014}
D.~P. Kingma and J.~Ba.
\newblock Adam: A method for stochastic optimization.
\newblock preprint, arXiv:1412.6980, 2014.

\bibitem{KurGooBen2016a}
A.~Kurakin, I.~J. Goodfellow, and S.~Bengio.
\newblock Adversarial examples in the physical world.
\newblock In {\em ICLR Workshop}, 2017.

\bibitem{LiuEtAl2016}
Y.~Liu, X.~Chen, C.~Liu, and D.~Song.
\newblock Delving into transferable adversarial examples and black-box attacks.
\newblock In {\em ICLR}, 2017.

\bibitem{MadEtAl2018}
A.~Madry, A.~Makelov, L.~Schmidt, D.~Tsipras, and A.~Valdu.
\newblock Towards deep learning models resistant to adversarial attacks.
\newblock In {\em ICLR}, 2018.

\bibitem{MirGehVec2018}
M.~Mirman, T.~Gehr, and M.~Vechev.
\newblock Differentiable abstract interpretation for provably robust neural
  networks.
\newblock In {\em ICML}, 2018.

\bibitem{MooFawFro2016}
S.-M. Moosavi-Dezfooli, A.~Fawzi, and P.~Frossard.
\newblock Deepfool: a simple and accurate method to fool deep neural networks.
\newblock In {\em CVPR}, pages 2574--2582, 2016.

\bibitem{MosEtAl18}
M.~Mosbach, M.~Andriushchenko, T.~Trost, M.~Hein, and D.~Klakow.
\newblock Logit pairing methods can fool gradient-based attacks.
\newblock In {\em NeurIPS 2018 Workshop on Security in Machine Learning}, 2018.

\bibitem{Cleverhans2017}
N.~Papernot, N.~Carlini, I.~Goodfellow, R.~Feinman, F.~Faghri, A.~Matyasko,
  K.~Hambardzumyan, Y.-L. Juang, A.~Kurakin, R.~Sheatsley, A.~Garg, and Y.-C.
  Lin.
\newblock cleverhans v2.0.0: an adversarial machine learning library.
\newblock preprint, arXiv:1610.00768, 2017.

\bibitem{PapEtAl2016a}
N.~Papernot, P.~McDonald, X.~Wu, S.~Jha, and A.~Swami.
\newblock Distillation as a defense to adversarial perturbations against deep
  networks.
\newblock In {\em IEEE Symposium on Security \& Privacy}, 2016.

\bibitem{RagSteLia2018}
A.~Raghunathan, J.~Steinhardt, and P.~Liang.
\newblock Certified defenses against adversarial examples.
\newblock In {\em ICLR}, 2018.

\bibitem{SokEtAl2016}
J.~Sokolic, R.~Giryes, G.~Sapiro, and M.~R.~D. Rodrigues.
\newblock Robust large margin deep neural networks.
\newblock {\em IEEE Transactions on Signal Processing}, 65:4265 – 4280, 2017.

\bibitem{GTSB2012}
J.~Stallkamp, M.~Schlipsing, J.~Salmen, and C.~Igel.
\newblock Man vs. computer: Benchmarking machine learning algorithms for
  traffic sign recognition.
\newblock {\em Neural Networks}, 32:323--332, 2012.

\bibitem{SzeEtal2014}
C.~Szegedy, W.~Zaremba, I.~Sutskever, J.~Bruna, D.~Erhan, I.~Goodfellow, and
  R.~Fergus.
\newblock Intriguing properties of neural networks.
\newblock In {\em ICLR}, pages 2503--2511, 2014.

\bibitem{TjeTed2017}
V.~Tjeng, K.~Xiao, and R.~Tedrake.
\newblock Evaluating robustness of neural networks with mixed integer
  programming.
\newblock preprint, arXiv:1711.07356v3, 2019.

\bibitem{TsiEtAl18}
D.~Tsipras, S.~Santurkar, L.~Engstrom, A.~Turner, and A.~Madry.
\newblock Robustness may be at odds with accuracy.
\newblock preprint, arXiv:1805.12152, 2018.

\bibitem{WengEtAl2018}
T.~Weng, H.~Zhang, H.~Chen, Z.~Song, C.~Hsieh, L.~Daniel, D.~S. Boning, and
  I.~S. Dhillon.
\newblock Towards fast computation of certified robustness for relu networks.
\newblock In {\em ICML}, 2018.

\bibitem{WonKol2018}
E.~Wong and J.~Z. Kolter.
\newblock Provable defenses against adversarial examples via the convex outer
  adversarial polytope.
\newblock In {\em ICML}, 2018.

\bibitem{WonEtAl18}
E.~Wong, F.~Schmidt, J.~H. Metzen, and J.~Z. Kolter.
\newblock Scaling provable adversarial defenses.
\newblock In {\em NeurIPS}, 2018.

\bibitem{XiaoEtAl2017}
H.~Xiao, K.~Rasul, and R.~Vollgraf.
\newblock Fashion-{MNIST}: a novel image dataset for benchmarking machine
  learning algorithms.
\newblock preprint, arXiv:1708.07747, 2017.

\bibitem{XiaoEtAl18}
K.~Y. Xiao, V.~Tjeng, N.~M. Shafiullah, and A.~Madry.
\newblock Training for faster adversarial robustness verification via inducing
  relu stability.
\newblock preprint, arXiv:1809.03008, 2018.

\bibitem{ZheEtAl2016}
S.~Zheng, Y.~Song, T.~Leung, and I.~J. Goodfellow.
\newblock Improving the robustness of deep neural networks via stability
  training.
\newblock In {\em CVPR}, 2016.

\end{thebibliography}
\let\oldbibliography\thebibliography
\renewcommand{\thebibliography}[1]{%
  \oldbibliography{#1}%
  \setlength{\itemsep}{0pt}%
}

\clearpage

\ifappendix
\appendix
\section{Improving lower bounds}\label{app:improving_bounds}

\subsection{Integration of box constraints into robustness guarantees}\label{app:postproc}
Note that in many applications the input space is not full $\R^d$ but a certain subset $C$ due to constraints e.g. images
belong to $C=[0,1]^d$. These constraints typically increase the norm of the minimal perturbation in \eqref{eq:advopt}. Thus it is important to integrate
the constraints in the generation of adversarial samples (upper bounds) as done e.g. in \cite{CarWag2016}, but obviously we should also integrate them in
the computation of the lower bounds which is in our case based on the computation of distances to hyperplanes. The computation of 
the $l_p$-distance of $y$ to a hyperplane $(w,b)$ has a closed form solution:
\begin{align}\label{eq:dist-hyperplane2}
\min\{\norm{y-x}_p \,|\, \inner{w,x}+b=0\} = \frac{|\inner{w,y}+b|}{\norm{w}_q},
\end{align}
The additional box constraints lead to the following optimization problem,
\begin{align}\label{eq:dist-hyperplane-const2} 
\min\{\norm{y-x}_p \,|\, \inner{w,x}+b=0,\quad x \in [0,1]^d\},
\end{align}
which is convex but has no analytical solution. However, its dual is just a one-dimensional convex optimization problem which can be solved
efficiently. In fact a reformulation of this problem has been considered in \cite{HeiAnd2017}, where fast solvers for $p \in \{1,2,\infty\}$ are proposed.
Moreover, when computing the distance to the boundary of the polytope or the decision boundaries one does not need to solve always the box-constrained
distance problem \eqref{eq:dist-hyperplane-const2}. It suffices to compute first the distances \eqref{eq:dist-hyperplane2} as they are smaller or equal
to the ones of \eqref{eq:dist-hyperplane-const2} and sort them in ascending order. Then one computes the box-constrained distances in the given
order and stops when the smallest computed box-constrained distance is smaller than the next original distance in the sorted list. In this way one typically just needs to solve a very small fraction of all box-constrained problems.  The integration of the box constraints is important as the lower bounds improve on average by 20\% and this can make the difference between having a certified optimal solution and just a lower bound.

\subsection{Checking neighboring regions} \label{subsection:check_neighb_regions}
In order to improve the lower bounds we can use not only the linear region $Q(x)$ where $x$ lies but also some neighboring regions. The following description is just a sketch as one has to handle several case distinctions.

Let $x$ be the original point and $H=\{\pi_1,...,\pi_n\}$ the set of hyperplanes defining the polytope $Q(x)$ sorted so that $d_C(x,\pi_i)< d_C(x,\pi_j)$ if $i<j$, where $d_C$ is the distance including box constraints. If we do not directly get the guaranteed optimal solution, we get an upper bound ($u$, namely the distance to the decision boundary inside $Q(x)$) and a lower bound for the norm of the adversarial perturbation ($l=d_C(x,\pi_1)$). If $l<u$, we can check the region that we find on the other side of $\pi_1$. In order to get the corresponding description of the polytope on the other side, we just have to change the corresponding entry in the activation matrix $\Sigma$ of the layer where $\pi_1$ belongs to and recompute the hyperplanes of the new linear region $R$. Solving \eqref{eq:advopt} on the second region we get a new upper bound if the distance of $x$ to the decision boundary in $R$ is smaller than $u$. Moreover we update $H$ with the hyperplanes given by the second region. Finally, if $u<d_C(x,\pi_2)$ then $u$ is the optimal solution, otherwise $l=d_C(x,\pi_2)$ and we can repeat this process with the next closest hyperplane. At the moment we stop after checking maximally $5$ neighboring linear regions.

\section{Main experiments} \label{app:main_exp}
\subsection{Experimental details}
By FC1 we denote a one hidden layer fully connected network with 1024 hidden units. By FC10 we denote a 10 hidden layers network that has 1 layer with 124 units, seven layers with 104 units and two layers with 86 units (so that the total number of units is again 1024). The convolutional architecture that we use is identical to \cite{WonKol2018}, which consists of two convolutional layers with 16 and 32 filters of size $4 \times 4$ and stride 2, followed by a fully connected layer with 100 hidden units.
For all experiments we use batch size 128 and we train the models for 100 epochs. Moreover, we use Adam optimizer \cite{KinEtAl2014} with the default learning rate 0.001 for all models except for the $l_2$ models on MNIST and F-MNIST where we use the learning rate of $10^{-4}$ for MMR and $5 \cdot 10^{-5}$ for MMR+at. We also reduce the learning rate by a factor of 10 for the last 10 epochs. On CIFAR-10 dataset we apply random crops and random mirroring of the images as data augmentation.\\
For training we use MMR regularizer in the formulation \eqref{eq:regularizer_bottom_k} with $k_D$ equal to the number of classes, and with $k_B$ linearly (wrt the epoch) decreasing from 10\% to 2\% of the total number of hidden units of the particular network architecture. We also use a training schedule for $\lambda$ where we linearly increase it from $\lambda / 10$ to $\lambda$ during the first 10 epochs. We employ both schemes since they increase the stability of training with MMR. \\
In order to determine the best set of hyperparameters $\lambda$, $\gamma_B$, and $\gamma_D$ of MMR, we perform a grid search over them for every dataset and network architecture. In particular, we empirically found that the optimal values of $\gamma_B$ and $\gamma_D$ are equal and are usually 1.5-2 times higher than the $\epsilon$ of robust error. All the reported MMR models and the final set of hyperparameters can be found at \url{https://github.com/max-andr/provable-robustness-max-linear-regions}. \\
In order to make a comparison to the robust training of \cite{WonKol2018} we either take their publicly available models or retrain them using their code. For the main experiments, we set the parameter for KW training equal to the $\epsilon$ for which we check the robust error. For experiments where robustness is evaluated as average lower bound (see Section \ref{sec:further_exp}), we performed a grid search over the radius of the $l_2$-norm used in their robust loss, aiming at a model with non-trivial lower bounds with little or no loss in test error.

We perform adversarial training using the PGD attack of \cite{MadEtAl2018} with 50\% clean and 50\% adversarial examples in every batch. For the $l_2$-norm, we adapted the implementation from \cite{Cleverhans2017} to perform the gradient update normalized by its $l_2$ norm, instead of the gradient sign (which corresponds to $l_{\infty}$-norm and thus irrelevant for $l_2$ case) on every iteration. We use the same $l_2$-bound on the perturbation as the $\epsilon$ used in robust error. During training, we perform 40 iterations of the PGD attack for MNIST and F-MNIST, and 7 for GTS and CIFAR-10. During evaluation, we use 40 iterations for all datasets. The step size is selected as $\epsilon$ divided by the number of iterations and multiplied by 2. \\
We use the untargeted formulation of the Carlini-Wagner $l_2$ attack in order to evaluate the upper bounds on the $l_2$-norm required to change the class. We use the settings provided in the original paper \cite{CarWag2016} and in their code, including 20 iterations of the binary search, 10000 iterations of the optimizer, learning rate 0.01 and initial constant of 0.001.
For the Mixed Integer Programming evaluation we use the library of \cite{TjeTed2017} with the settings of \cite{XiaoEtAl18}, which we obtained via private correspondence. Namely, we use Gurobi as the back-end, LP as the tightening algorithm, and we set 5s timeout for the presolver, and 120s timeout for the main solver.

\ifpaper \else
\begin{figure}[t]\centering\includegraphics[scale=0.5]{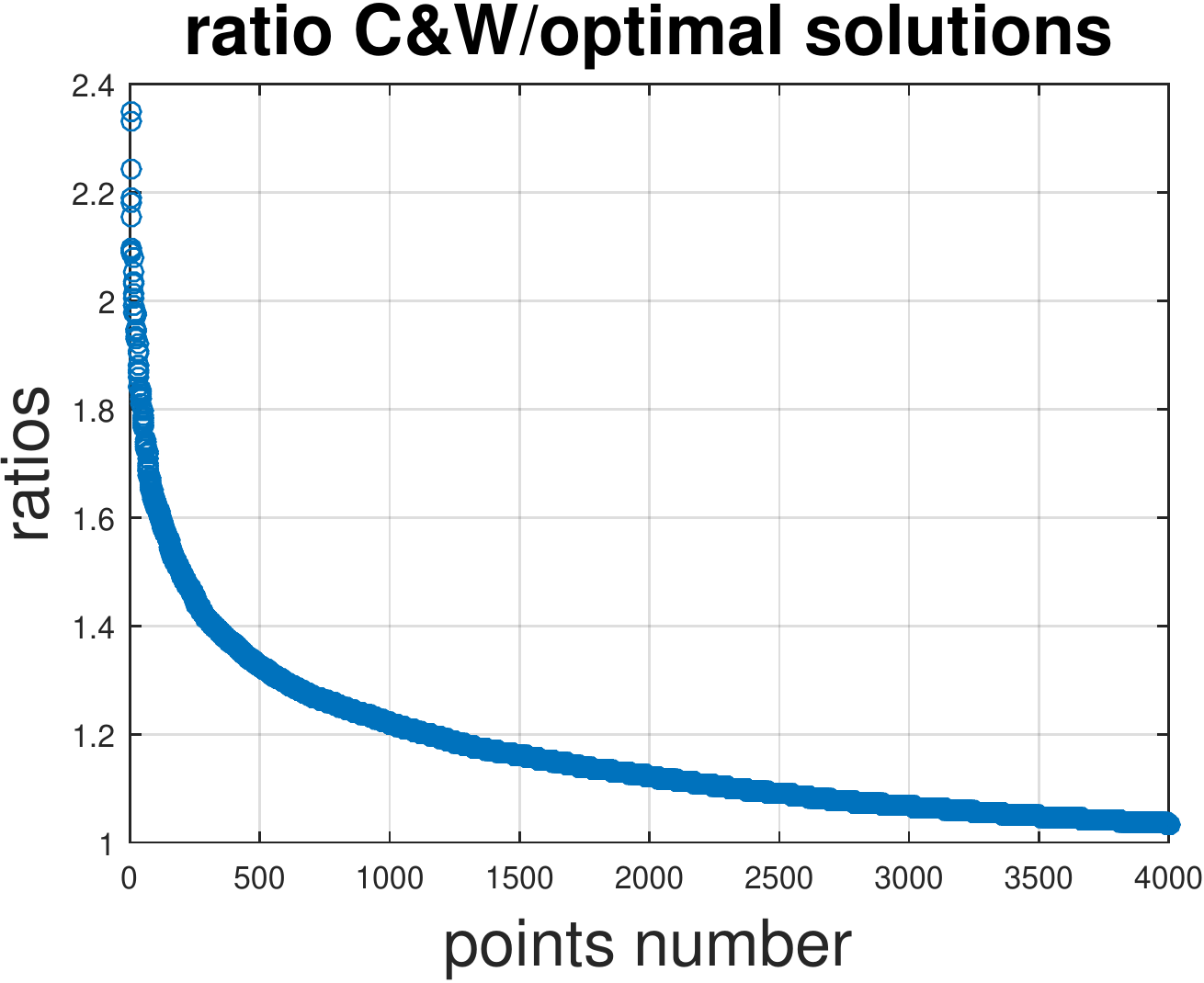}\caption{\textbf{Evaluation of CW-attack.} We report the descending sorted ratios $\norm{\delta_{CW}}_2$/$\norm{\delta_{opt}}_2$ (norm of the outcomes of CW-attack divided by the norm of minimal adversarial examples) with regard to a model on GTS dataset trained with our regularizer.} \label{fig:CW_long2}
\end{figure}
\fi

\begin{table*}[p]
	\centering
	\caption{\textbf{Lower bounds computed by our method.} We report here for the fully connected models trained with either MMR or MMR+\textit{at} the lower bounds computed by our technique, that is exploiting Theorem \ref{th:main} and integrating box constraints without and with checking additional neighboring regions (improved lower bounds) versus KW \cite{WonKol2018}.}
	\begin{tabular}{l l l || c | c | c |c }
		\multicolumn{4}{c|}{} & KW \cite{WonKol2018} & Theorem \ref{th:main} & our improved \\
		dataset & model & & test error & lower bounds  & lower bounds & lower bounds  \\
		\hline
		\multirow{4}{*}{MNIST}& FC1 & MMR &1.51\% & 0.69 & 0.22 & 0.29\\
		&FC1 & MMR+at &1.59\% & 0.70 & 0.25 & 0.33 \\
		& FC10 & MMR& 1.87\% & 0.48 & 0.31 & 0.33 \\
		& FC10 & MMR+at & 1.35\% & 0.40 & 0.21 & 0.26 \\
		\hline
		\multirow{4}{*}{GTS}& FC1 & MMR & 11.15\% & 0.69 & 0.69 & 0.69  \\
		& FC1 & MMR+at & 11.72\% & 0.72 & 0.72 & 0.72 \\
		& FC10 & MMR & 12.82\% & 0.64 & 0.63 &0.63 \\	
		& FC10 & MMR+at & 13.36\% & 0.64 & 0.63 &0.63 \\	
		\hline
		\multirow{4}{*}{F-MNIST}& FC1 &MMR & 10.22\% & 0.50 & 0.30 & 0.41 \\
		& FC1 &MMR+at & 10.94\% & 0.66 & 0.33 & 0.42 \\
		& FC10 & MMR & 11.73\% &0.68 & 0.56 & 0.64 \\
		& FC10 & MMR+at &11.39\% & 0.67& 0.53 & 0.60 
	\end{tabular}
	\label{tab:our_bounds2}
\end{table*}

\begin{table*}[p]
	\centering
	\caption{\textbf{Comparison of different methods regarding robustness wrt $l_2$-norm.} We here report the statistics of 5 different training schemes: \textit{plain} (usual training), \textit{at} (adversarial training \cite{MadEtAl2018}), MMR (ours), MMR+\textit{at} (MMR plus adversarial training) and KW (the robust training introduced in \cite{WonKol2018}). We show test error (TE), average of lower (LB) and upper (UB) bounds on the robustness $\norm{\delta}_2$, where $\delta$ is the solution of \eqref{eq:advopt}. The robustness statistics are computed on the first 1000 points of the respective test sets (including misclassified images) against all the possible target classes.}
	\begin{tabular}{l l | r r r | r r r | r r r}
		\multicolumn{11}{c}{\textbf{$l_2$-norm robustness}}\\ [4mm]
		\multirow{2}{*}{\textit{dataset}}&\textit{training}& \multicolumn{3}{c|}{FC1} & \multicolumn{3}{c|}{FC10}& \multicolumn{3}{c}{CNN} \\
		\cline{3-11}
		&\textit{scheme} & TE(\%) & LB & UB & TE(\%) & LB & UB & TE(\%) & LB & UB\\
		\hline
		\multirow{5}{*}{MNIST}&plain &1.59 & 0.34 &0.98 &1.81 & 0.13 & 0.70 & 0.97 & 0.04 & 1.03 \\
		&at & 1.29 & 0.25 & 1.23 & 0.93 & 0.14 & 1.59 & 0.86 & 0.14 & 1.67\\
		&KW & 1.37 & \textbf{0.70} & \textbf{1.75} & 1.69 & \textbf{0.75} & \textbf{1.74} & 1.04 & 0.32 & 1.84 \\
		\cdashline{2-11}
		&MMR & 1.51 & 0.69 & 1.69 & 1.87 & 0.48 & 1.48 &1.17 & \textbf{0.38} & 1.70 \\
		&MMR+at & 1.59 & \textbf{0.70} & 1.70 & 1.35 & 0.40 & 1.60 & 1.14 & \textbf{0.38} & \textbf{1.86} \\
		\hline
		\multirow{5}{*}{GTS}&plain & 12.24& 0.33 & 0.57 &11.25 & 0.08 & 0.48 & 6.73 & 0.06 & 0.43 \\
		&at & 13.55 &0.34 &0.66 & 13.01 & 0.10 & 0.56 & 8.12 & 0.06 &0.53 \\
		&KW & 13.06 &0.35 &0.63 & 13.56 & 0.16 & 0.52 & 8.44 & \textbf{0.11} & 0.52 \\
		\cdashline{2-11}
		&MMR & 11.15& 0.69 & 0.69  & 12.82 & \textbf{0.64} & \textbf{0.67} & 7.40 & 0.09 & 0.59 \\
		&MMR+at & 11.72 &\textbf{0.72} &\textbf{0.72} & 13.36 & \textbf{0.64} & 0.66 & 10.50 & \textbf{0.11} & \textbf{0.62}\\
		\hline
		\multirow{5}{*}{F-MNIST}&plain & 9.61 & 0.18 & 0.53 & 10.53 & 0.05 & 0.44 &8.86 &0.03 & 0.32 \\
		&at & 9.89 & 0.11 & 1.00& 9.89 & 0.11 & 1.00 & 8.77 & 0.07 &0.80 \\
		&KW & 9.95 & 0.46 & 1.11 & 11.42 & 0.47 & 1.22 & 10.37 & 0.17 & 0.96 \\
		\cdashline{2-11}
		&MMR & 10.22 & 0.50 & 0.85 & 11.73 &\textbf{0.68}& 1.18 & 10.30 & 0.17 & 0.88\\
		&MMR+at & 10.94& \textbf{0.66}&\textbf{1.45} & 11.39 & 0.67 & \textbf{1.24} & 10.48 & \textbf{0.21} & \textbf{1.14}\\
		\hline
	\end{tabular}
	\label{tab:main_exps4_v1_2}
\end{table*}

\begin{table}[t]
	\centering
	\caption{\textbf{Comparison of different training schemes wrt $l_2$-norm.} We here report the statistics relative to models trained with 5 different training scheme: \textit{plain} (usual training), \textit{at} (adversarial training \cite{MadEtAl2018}), MMR (ours), MMR+\textit{at} (the two methods combined) and KW (the robust training introduced in \cite{WonKol2018}). We show test error, average of lower and upper bounds on $\norm{\delta}_2$, where $\delta$ is the solution of problem \eqref{eq:advopt}. The statistics are computed on the first 1000 points of the test set (including misclassified images) against all the possible target classes.}
	\begin{tabular}{ l |  r r r}
		\multicolumn{4}{c}{\textbf{$l_2$-norm robustness on CIFAR-10}}\\[4mm]
		\textit{training}& \multicolumn{3}{c}{CNN} \\
		\cline{2-4}
		\textit{scheme} & TE(\%) & LB & UB\\
		\hline
		plain & 25.98 & 0.02 & 0.16 \\
		at & 25.36 & 0.04 & 0.42\\
		KW & 41.52 & \textbf{0.16} & \textbf{0.66} \\
		\cdashline{1-4}
		MMR & 41.86 & \textbf{0.16} & 0.39 \\
		MMR+at & 41.11 & 0.13 & 0.57 \\
		\hline
	\end{tabular}
	\label{tab:bounds_cifar10_2}
\end{table}

\section{Further experiments} \label{app:further_exp}
\ifpaper \else
\textbf{Evaluation of CW-attack:} In Figure \ref{fig:CW_long2} we show the 4000 largest ratios $\nicefrac{\norm{\delta_{CW}}_2}{\norm{\delta_{opt}}_2}$, that is the norm of the outcomes of CW-attack divided by the norm of minimal adversarial examples, computed on the GTS test set for a model trained with our regularizer. We can see that the perturbations found by the attack are for a significant fraction of the images much larger than the minimal ones.
\fi

\textbf{Comparison to Cross-Lipschitz regularization:} We also consider models trained with Cross-Lipschitz regularization of \cite{HeiAnd2017} since they also consider the robustness wrt the $l_2$ norm. We evaluated upper bounds on the robust error of MNIST-FC1 models using the method of \cite{WonKol2018}, which gives tighter bounds than the original Cross-Lipschitz guarantee of \cite{HeiAnd2017}. As a result, their most provably robust model obtained test error of 1.38\%, and the robust error bounded between 2.5\% given by the PGD attack and 7.2\% by \cite{WonKol2018}. Thus we can observe that Cross-Lipschitz regularization also enhances certifiability in this case since the upper bound is significantly better than the one for adversarial training, 16.9\%. However, both KW and MMR+at provide a better upper bound on robust error, 5.2\% and 6.4\% respectively.

\textbf{Comparison of lower bounds:} In Table \ref{tab:our_bounds2} we compare, for fully connected models, the lower bounds on the distance to the decision boundary computed by \cite{WonKol2018} and our technique using Theorem \ref{th:main} with integration of box constraints once just checking the initial linear region $Q(x)$ where the point $x$ lies versus also checking neighboring linear regions. We see that \cite{WonKol2018} obtain better lower bounds, this is why we use their method for the main evaluation of provable robustness in Table~\ref{tab:main_exps4_2}. Nevertheless, the gap is not too large and while the lower bounds are worse, the achieved robustness using our MMR regularization is mostly better as shown in Table \ref{tab:main_exps4_2}.

\textbf{Analysing $l_2$-robustness with different metrics:} We want here to repeat the experiment of Section \ref{sec:exp} wrt $l_2$-norm but evaluating robustness as the average norm of the perturbation necessary to change the classification of a point.
We can compute for every input a lower bound on the $l_2$-norm of the minimal adversarial perturbation thanks to the method of \cite{WonKol2018} and an upper bound with CW-attack \cite{CarWag2016}. However, when our technique (Theorem \ref{th:main}) provides the optimal solution, we set both lower and upper bounds to this value.\\
We report test error of the model and the average lower and upper bounds in Tables \ref{tab:main_exps4_v1_2} and \ref{tab:bounds_cifar10_2}, computed on 1000 points of the test set. For KW, MMR and MMR + adversarial training we report the solutions which achieve similar test error than the plain model.\\
There are several interesting observations.
First of all, while adversarial training improves the upper bounds compared to the plain setting often quite significantly, the lower bounds almost never improve, often they get even worse. This is in contrast to the methods, KW and our MMR, which optimize the robustness guarantees. For MMR we see in all cases significant improvements of the lower bounds over plain and
adversarial training, for KW this is also true but the improvements on GTS are much smaller. Notably, for the fully connected networks FC1 and FC10 on GTS and F-MNIST , 
the \emph{lower bounds} achieved by MMR and/or MMR+at are \emph{larger} than the \emph{upper bounds} of the plain training for F-MNIST and better than plain and adversarial training as well as KW on GTS. Thus MMR is provably more robust than the competing methods in these configurations. Moreover, the achieved lower bounds of MMR are only worse than the ones of KW on MNIST for FC10. Also for the achieved upper bounds MMR is most of the time better than KW and always improves over adversarial training. For the CNNs the improvements of KW and MMR over plain and adversarial training in terms of lower and upper bounds are smaller than for the fully connected networks and it is harder to maintain similar test performance. The differences between
KW and MMR for the lower bounds are very small so that for CNNs both robust methods perform on a similar level.

\section{Visualizing the structure of provably robust models} Here we analyse the effect of MMR regularization on the gradient of cross entropy loss wrt the input and on the structure of the convolutional filters. We focus on models trained for $l_\infty$-robustness on MNIST, GTS, and CIFAR-10. We provide the plain and adversarially trained model as reference.\\
In Figures \ref{fig:grads_mnist}, \ref{fig:grads_gts}, \ref{fig:grads_cifar10} we visualize the gradients computed for 10 images of the corresponding test sets (shown in the first row) for the different training procedures. For MNIST, zero values are represented in white, negative in blue and positive in red, where every image is rescaled independently. For GTS and CIFAR-10, we follow \cite{TsiEtAl18} and clip the gradient values to $\pm$3 standard deviations and then rescale them to $[0, 1]^d$. We visualize models obtained with plain training (second row), adversarial training \cite{MadEtAl2018} (third), KW robust training \cite{WonKol2018} (fourth), MMR (fifth), MMR + adversarial training (last row). \\
Figures \ref{fig:filters_mnist}, \ref{fig:filters_gts}, \ref{fig:filters_cifar10} show the structure of the weights of the filters of the two convolutional layers. The top five rows are the filters from the first layer, either all of them for MNIST, or the first five for GTS and CIFAR-10 (reshaped from $4\times 4\times 3$ to $4\times 12$). The bottom five are the filters from the second convolution layer (reshaped from $4\times 4\times 16$ to $16\times 16$). We present the filters for plain training (first and sixth row), adversarial training \cite{MadEtAl2018} (second and seventh), KW robust training \cite{WonKol2018} (third and eighth), MMR (fourth and ninth), MMR + adversarial training (fifth and last row). The white color corresponds to zero weights, and the farther the value is from zero, the more intense is the color. Note that each row is also scaled independently.

\textbf{MNIST (Figures \ref{fig:grads_mnist}, \ref{fig:filters_mnist}):}
As noted in \cite{TsiEtAl18}, adversarial training leads to more interpretable gradients that focus on salient features of the digits. All robust training schemes have the effect of creating both interpretable and sparse gradients, which is reasonable considering that a sparse gradient implies that there are less directions along which abrupt variations in the value of the loss are possible. While KW model seems to highlight more the borders of the images, MMR models have the most concentrated gradients.\\
In line with what has been reported in \cite{MadEtAl2018,WonKol2018}, the filters of adversarially trained and provably robust models (Figure \ref{fig:filters_mnist}) have significantly higher sparsity than the plain model. Interestingly, in contrast to the others, MMR models preserve sparsity also in the filters on the second convolutional layer. This seems to be important for obtaining tight certificates by the KW method and for fast verification with the MIP solver.

\textbf{GTS (Figures \ref{fig:grads_gts}, \ref{fig:filters_gts}):} 
For GTS we also observe that all robust training schemes lead to a qualitatively different behaviour of the gradient. In particular, the models become less sensitive to variations of the background, and instead they highlight more the traffic signs. We observe that the gradients for the MMR and adversarially trained models are the most interpretable. We note that the similarity between the gradients of the KW model and MMR+at is possibly due to their similar test error and robustness (Table \ref{tab:main_exps4_2}).\\
Similarly to MNIST, we observe that the convolutional filters are sparser for the adversarially trained model than for the plain model. The filters of KW and MMR models are the sparsest, while MMR+at is less sparse and resembles more the adversarially trained model.

\textbf{CIFAR-10 (Figures \ref{fig:grads_cifar10}, \ref{fig:filters_cifar10}):} First of all, we note that the considered shallow CNNs are not able to achieve very low clean test error (the best is 24.63\% for the plain model and 34.61\% for provably robust models). This may explain why we cannot observe  interpretable gradients as clearly as in \cite{TsiEtAl18} even for the adversarially trained model. However, we note that the gradients for the robustly trained models are still qualitatively different from the plain model. While the gradients of the plain model just consist of high-frequency noise, e.g. the MMR model concentrates more on the objects rather than on the background. \\
For the convolutional filters of the CIFAR-10 models, we make conclusions similar to GTS. In particular, the filters of KW and MMR models are the sparsest, while MMR+at is less sparse and is more similar to the adversarially trained model.

\newcommand{\multicell}[2][c]{%
	\begin{tabular}[#1]{@{}c@{}} #2 \end{tabular}}

\begin{figure*}[p]
	\begin{tabular}{cc}
		\small
		\multicell[b]{Original \\ image \\ \\ \\ \\ plain \\ \\ \\ \\ at \\ \\ \\ \\ KW \\ \\ \\ \\ MMR \\ \\ \\ \\ MMR+at \\ \\ \\ \\} & \centering\hspace{-3mm}\includegraphics[width=1.9\columnwidth]{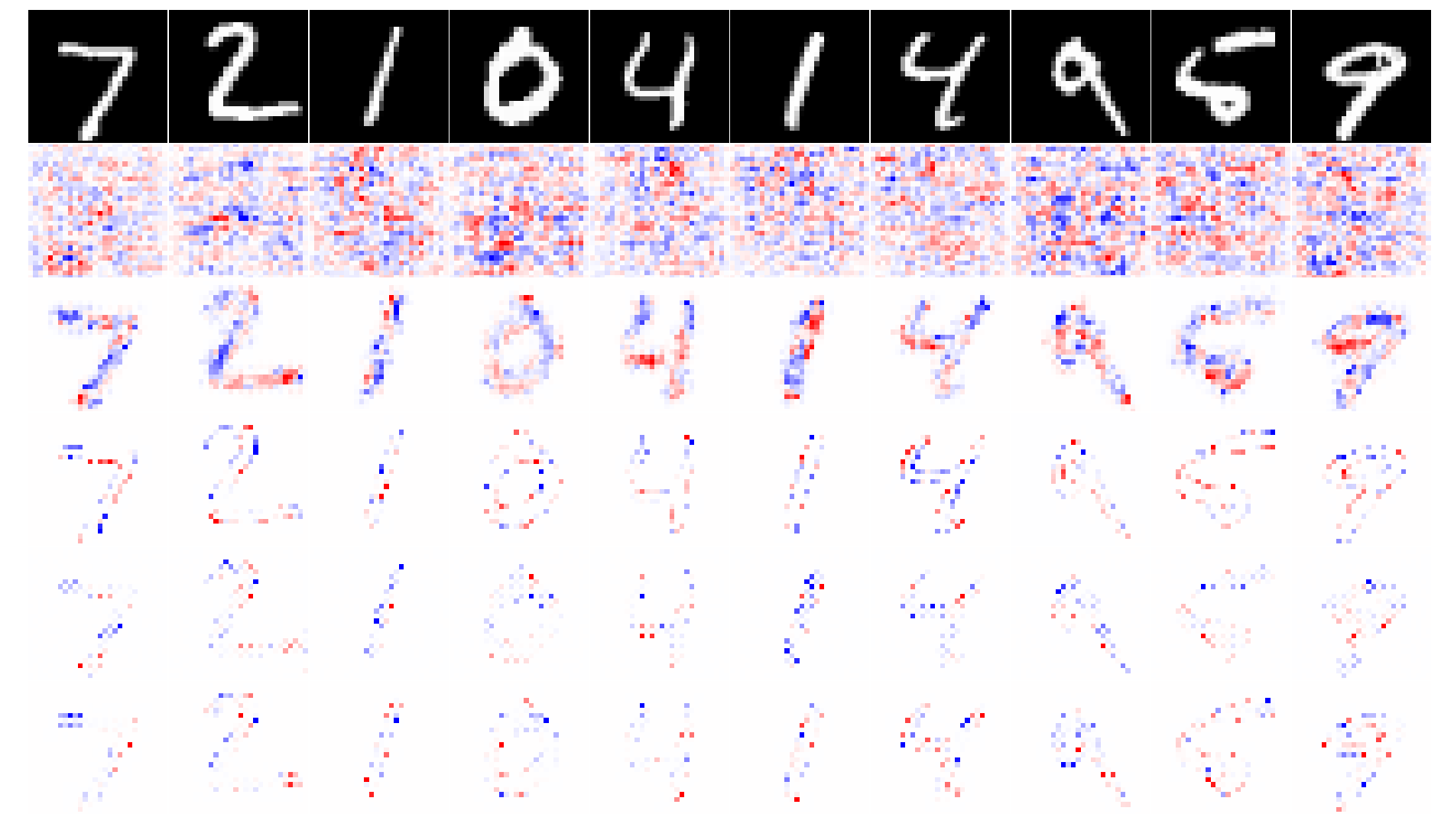} 
	\end{tabular}

	\caption{\textbf{Gradient of $l_\infty$-robust models.} We visualize the gradients of the cross entropy loss wrt the input for different images of MNIST test set for every model: plain training, adversarial training \cite{MadEtAl2018}, KW robust training \cite{WonKol2018}, MMR, MMR + adversarial training. We can see for robust models the gradients are much sparser, while only for plain training it does not clearly highlights relevant features.} \label{fig:grads_mnist}
\end{figure*}

\begin{figure*}[p]
	\begin{tabular}{cc}
		\small
		\multicell[b]{plain \\ \vspace{2.5mm} \\ at \\ \vspace{2.5mm} \\ KW \\ \vspace{2.5mm} \\ MMR \\ \vspace{2.5mm} \\ MMR+at \\ \vspace{1mm}} &
		\centering\includegraphics[width=1.8\columnwidth]{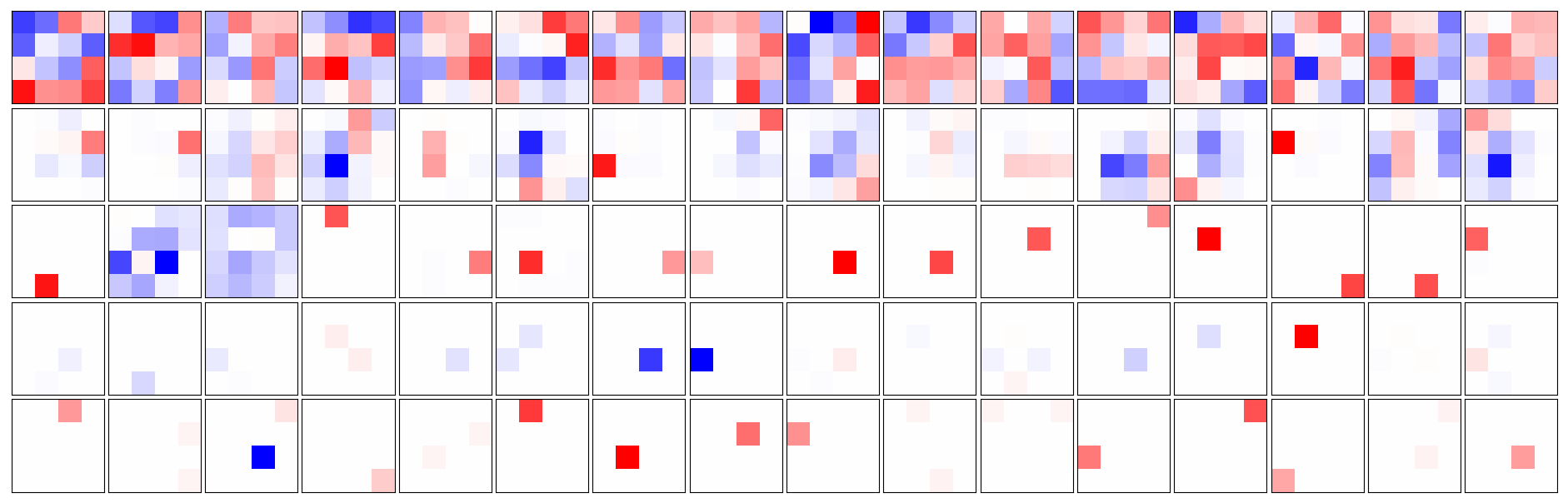}
	\end{tabular}
	\\
	\begin{tabular}{cc}
		\small
		\multicell[b]{plain \\ \vspace{2.5mm} \\ at \\ \vspace{2.5mm} \\ KW \\ \vspace{2.5mm} \\ MMR \\ \vspace{2.5mm} \\ MMR+at \\ \vspace{1mm}} &
		\centering\includegraphics[width=1.8\columnwidth]{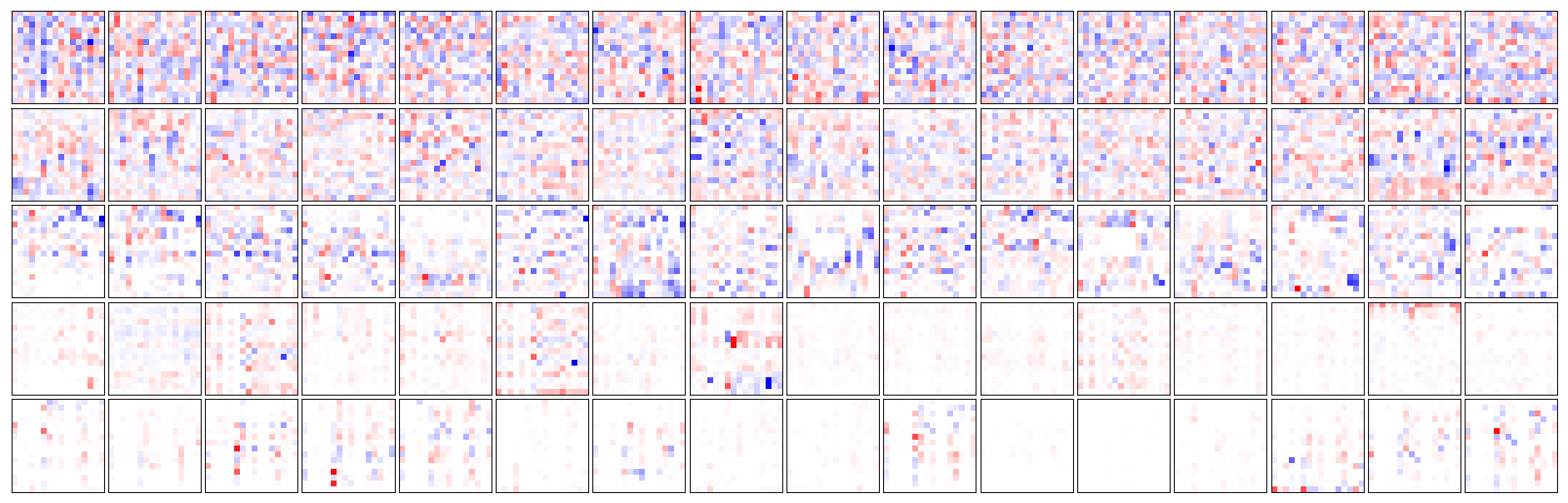}
	\end{tabular}
	\caption{\textbf{Filters of $l_\infty$-robust models.} We visualize all the filters of the first convolutional layer (first five rows) and 16 filters of the second layer (last five rows) for every MNIST model (trained for $l_\infty$-robustness): plain training, adversarial training \cite{MadEtAl2018}, KW robust training \cite{WonKol2018}, MMR, MMR + adversarial training. We can see that MMR and MMR+at leads to very sparse filters, especially in the second layer. Note that each row is rescaled independently.} \label{fig:filters_mnist}
\end{figure*}

\begin{figure*}[p]
	\begin{tabular}{cc}
		\small
		\multicell[b]{Original \\ image \\ \vspace{8mm} \\ plain \\ \vspace{7.5mm} \\ at \\ \vspace{7.5mm} \\ KW \\ \vspace{7.5mm} \\ MMR \\ \vspace{7.5mm} \\ MMR+at \\ \vspace{6mm}} & \centering\hspace{-3mm}\includegraphics[width=1.85\columnwidth]{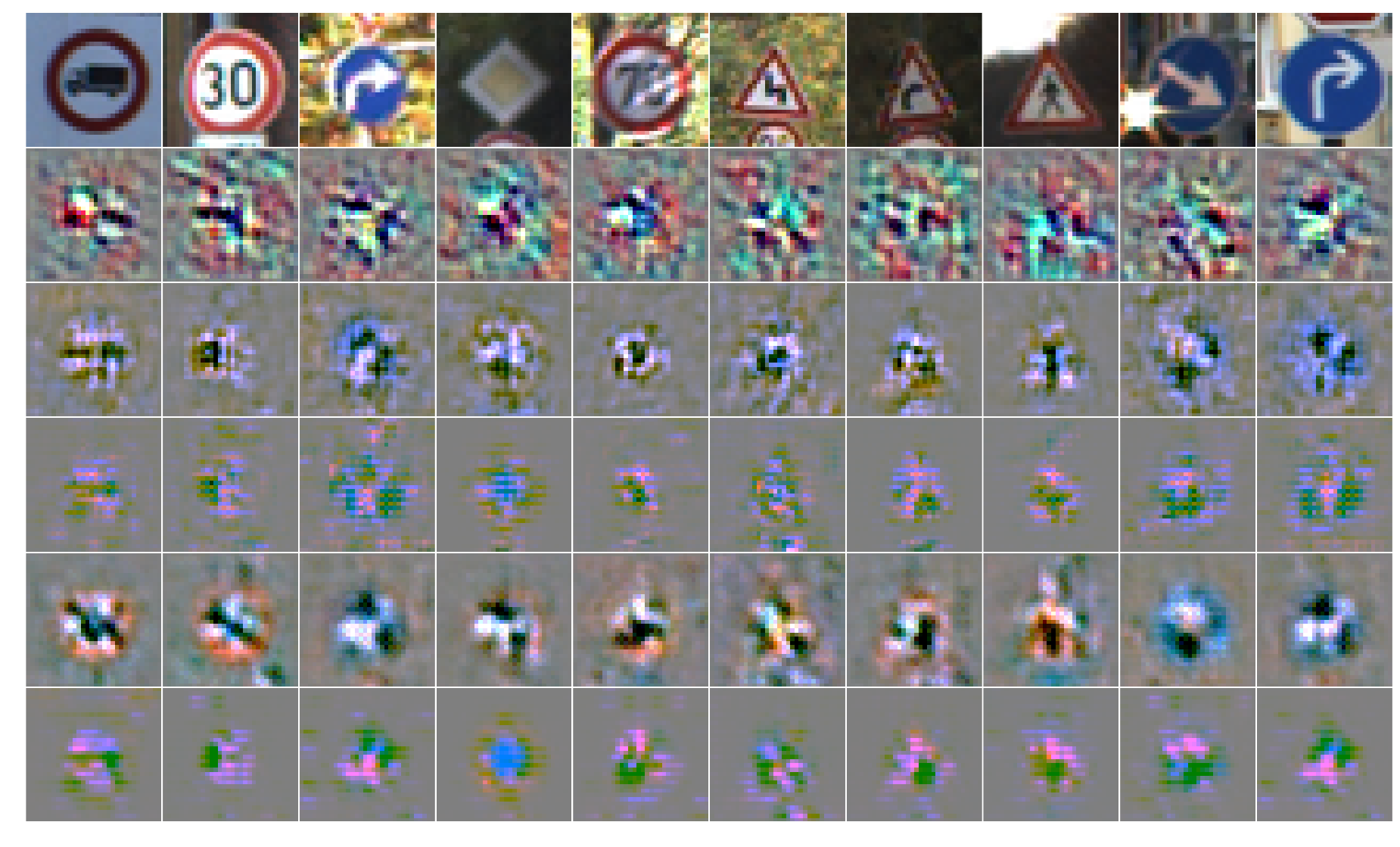} 
	\end{tabular}
	 \caption{\textbf{Gradient of $l_\infty$-robust models.} We visualize the gradients of the cross entropy loss wrt the input for different images of GTS test set (first row) for every model: plain training, adversarial training \cite{MadEtAl2018}, KW robust training \cite{WonKol2018}, MMR, MMR + adversarial training. We can see for robust models the gradients are much sparser, while only for plain training it does not clearly highlights relevant features.} \label{fig:grads_gts}
\end{figure*}

\begin{figure*}[p]
	\begin{tabular}{cc}
		\small
		\multicell[b]{plain \\ \vspace{2.5mm} \\ at \\ \vspace{2.5mm} \\ KW \\ \vspace{2.5mm} \\ MMR \\ \vspace{2.5mm} \\ MMR+at \\ \vspace{1mm}} &
		\centering\includegraphics[width=1.55\columnwidth]{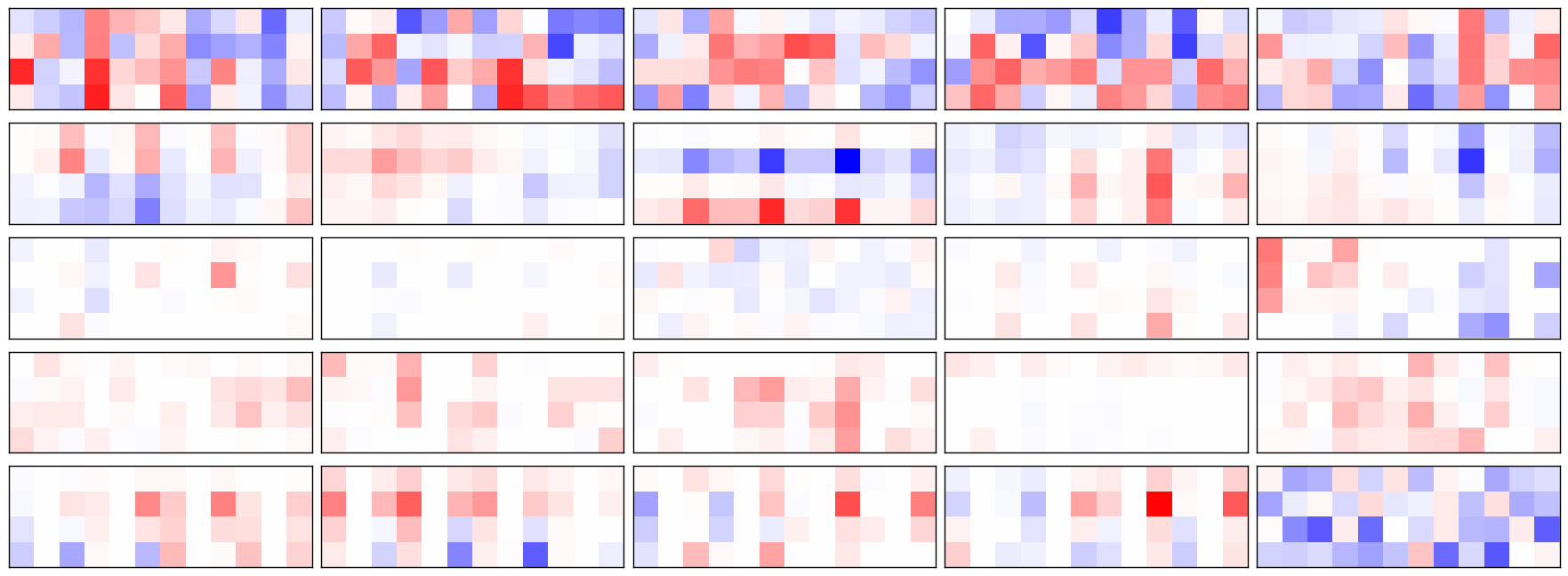}
	\end{tabular}
	\\
	\begin{tabular}{cc}
		\small
		\multicell[b]{plain \\ \vspace{2.5mm} \\ at \\ \vspace{2.5mm} \\ KW \\ \vspace{2.5mm} \\ MMR \\ \vspace{2.5mm} \\ MMR+at \\ \vspace{1mm}} &
		\centering\includegraphics[width=1.8\columnwidth]{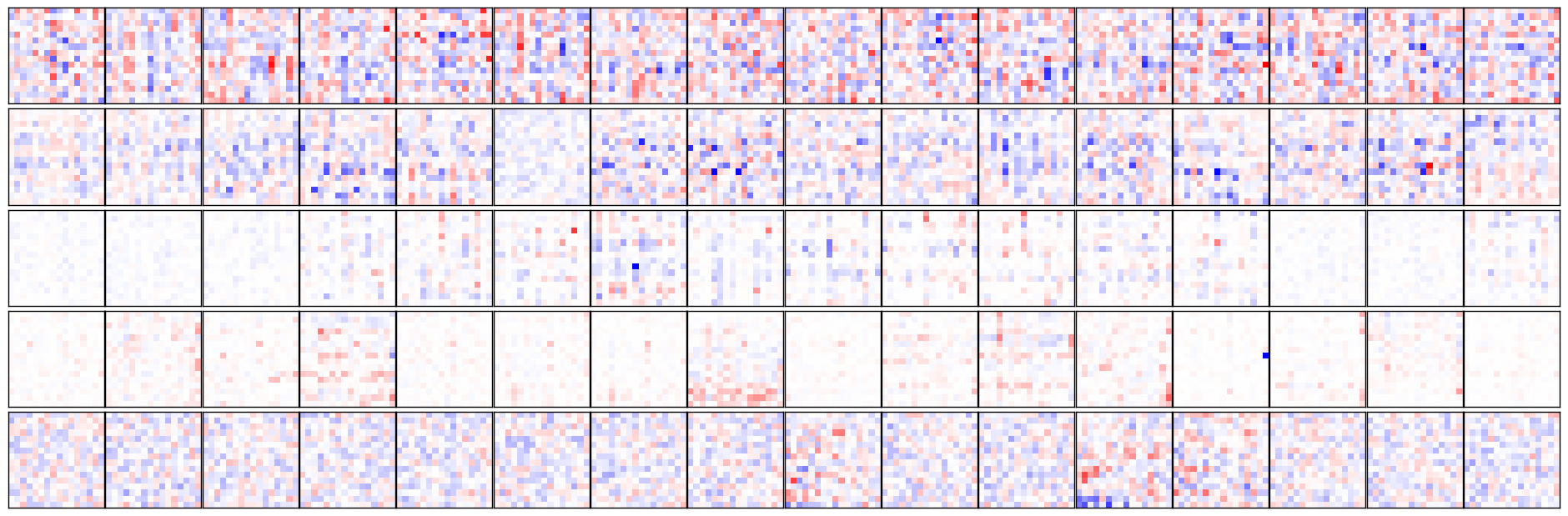}
	\end{tabular}
	\caption{\textbf{Filters of $l_\infty$-robust models.} We visualize the first five filters of the first convolutional layer (first five rows) and 16 filters of the second layer (last five rows) for every GTS model (trained for $l_\infty$-robustness): plain training, adversarial training \cite{MadEtAl2018}, KW robust training \cite{WonKol2018}, MMR, MMR + adversarial training. We can see that MMR leads to sparser filters, especially in the second layer. Note that each row is rescaled independently.} \label{fig:filters_gts}
\end{figure*}

\begin{figure*}[p]
	\begin{tabular}{cc}
		\small
		\multicell[b]{Original \\ image \\ \vspace{8mm} \\ plain \\ \vspace{7.5mm} \\ at \\ \vspace{7.5mm} \\ KW \\ \vspace{7.5mm} \\ MMR \\ \vspace{7.5mm} \\ MMR+at \\ \vspace{6mm}} & \centering\hspace{-3mm}\includegraphics[width=1.85\columnwidth]{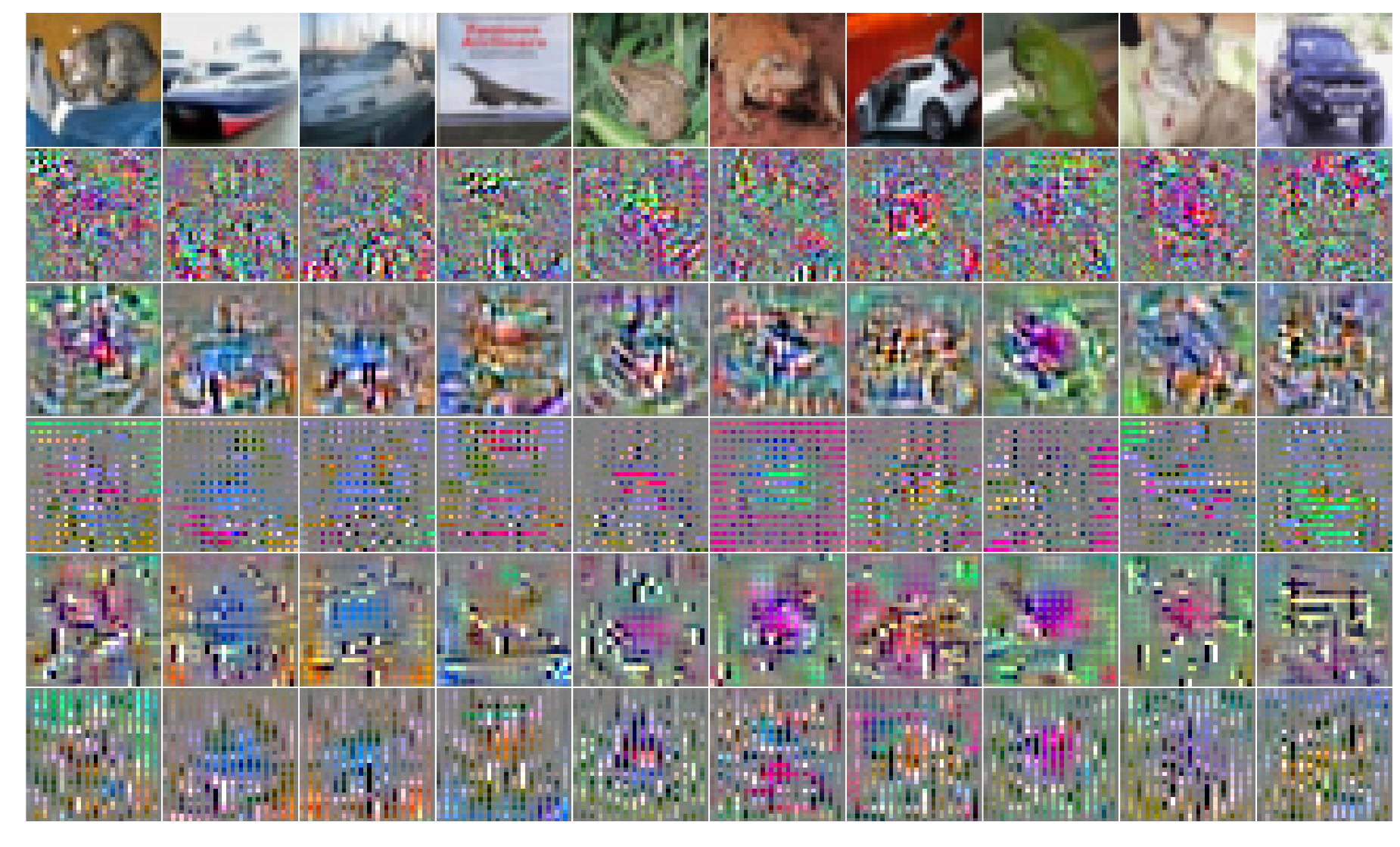} 
	\end{tabular}
	\caption{\textbf{Gradient of $l_\infty$-robust models.} We visualize the gradients of the cross entropy loss wrt the input for different images of CIFAR-10 test set for every model: plain training, adversarial training \cite{MadEtAl2018}, KW robust training \cite{WonKol2018}, MMR, MMR + adversarial training. We can see for robust models the gradients are sparser than for plain training and they tend to highlight relevant features of the images.} \label{fig:grads_cifar10}
\end{figure*}

\begin{figure*}[p]
	\begin{tabular}{cc}
		\small
		\multicell[b]{plain \\ \vspace{2.5mm} \\ at \\ \vspace{2.5mm} \\ KW \\ \vspace{2.5mm} \\ MMR \\ \vspace{2.5mm} \\ MMR+at \\ \vspace{1mm}} &
		\centering\includegraphics[width=1.55\columnwidth]{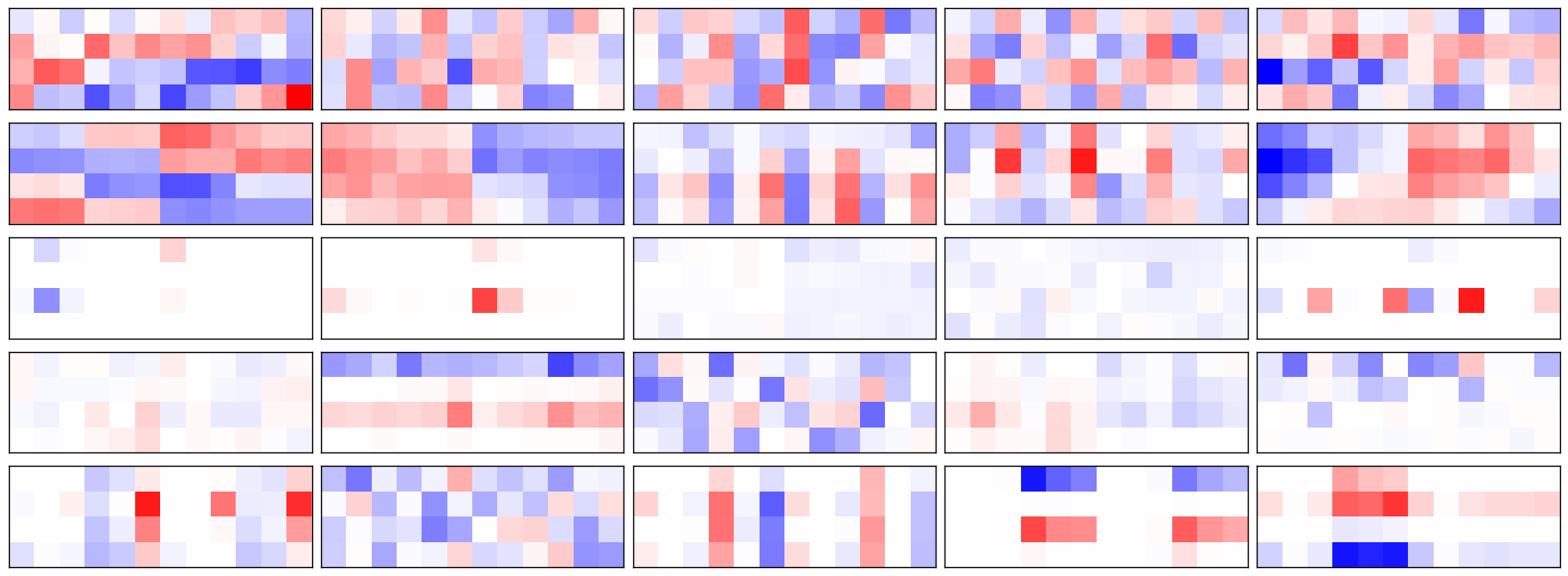}
	\end{tabular}
	\\
	\begin{tabular}{cc}
		\small
		\multicell[b]{plain \\ \vspace{2.5mm} \\ at \\ \vspace{2.5mm} \\ KW \\ \vspace{2.5mm} \\ MMR \\ \vspace{2.5mm} \\ MMR+at \\ \vspace{1mm}} &
		\centering\includegraphics[width=1.8\columnwidth]{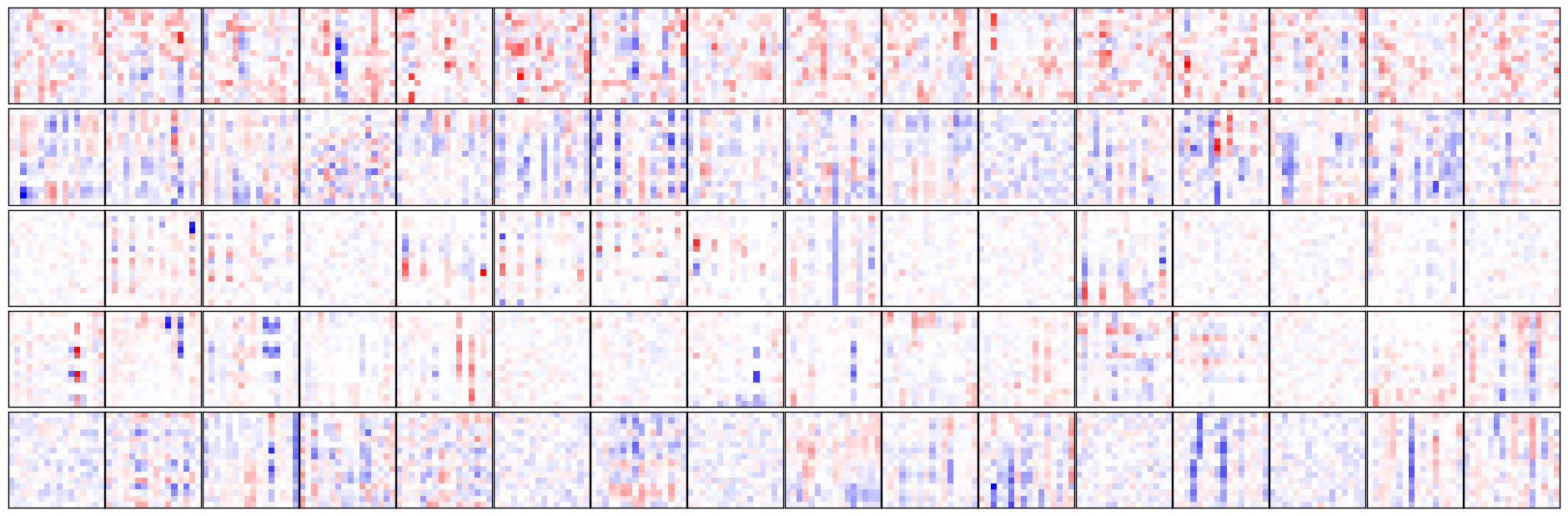}
	\end{tabular}
	\caption{\textbf{Filters of $l_\infty$-robust models.} We visualize the first five filters of the first convolutional layer (first five rows) and 16 filters of the second layer (last five rows) for every CIFAR-10 model (trained for $l_\infty$-robustness): plain training, adversarial training \cite{MadEtAl2018}, KW robust training \cite{WonKol2018}, MMR, MMR + adversarial training. We can see that MMR and KW lead to sparser filters, especially in the second layer. Note that each row is rescaled independently.} \label{fig:filters_cifar10}
\end{figure*}

\fi

\end{document}